%% file: main_arxiv.tex
\documentclass[11pt, dvipsnames, DIV=12, headings=medium]{scrartcl}
\usepackage[english]{babel}
\usepackage[T1]{fontenc}
\usepackage[utf8]{inputenc}
\usepackage{graphicx}
\usepackage{natbib}
\usepackage{adjustbox}
\usepackage{subcaption}
\usepackage{booktabs}
\usepackage{multirow}

\usepackage{url}
\usepackage{multirow}
\usepackage{tikz}
\usetikzlibrary{calc}
\usetikzlibrary{decorations.pathmorphing}
\usetikzlibrary{positioning}
\usetikzlibrary{shapes}
\usetikzlibrary{arrows.meta,backgrounds,automata}
\usetikzlibrary{positioning,shapes,matrix,calc}
\tikzstyle{vertex}=[circle, draw, fill=gray!80!white,thick,scale=1.2]
\tikzstyle{edge}=[draw=black, thick,-]
\usetikzlibrary{hobby,arrows,decorations.pathmorphing,backgrounds,positioning,fit,petri,calc}
\pgfdeclarelayer{background}
\pgfsetlayers{background,main}
\usepackage{csquotes}
\usepackage{paralist}

\usepackage{ifthen}
\newcommand{\CC}[1][]{$\text{C\hspace{-.25ex}}^{_{_{_{++}}}}
	\ifthenelse{\equal{#1}{}}{}{\text{\hspace{-.625ex}#1}}$}

\usepackage[framemethod=tikz]{mdframed}

\definecolor{mycolor}{rgb}{0.122, 0.435, 0.698}

\newmdenv[innerlinewidth=0.5pt, roundcorner=4pt,linecolor=mycolor,innerleftmargin=6pt,
innerrightmargin=6pt,innertopmargin=6pt,innerbottommargin=6pt]{mybox}

\usepackage{bm}

\newcommand{\oms}{\{\!\!\{}
\newcommand{\cms}{\}\!\!\}}

\usepackage{amsmath}
\usepackage{amssymb}
\usepackage{amsthm}
\usepackage{amsfonts}
\usepackage{thmtools}		
\usepackage{mleftright}
\usepackage{stmaryrd}
\usepackage{nicefrac}
\usepackage{algorithm}
\usepackage{algorithmicx}
\usepackage[noend]{algpseudocode}

\theoremstyle{definition}
\newtheorem{theorem}{Theorem}

\newtheorem{proposition}[theorem]{Proposition}

\newtheorem{observation}{Observation}

\newtheorem{lemma}[theorem]{Lemma}

\newtheorem{definition}[theorem]{Definition}

\usepackage{thm-restate}
\usepackage[mathic=true]{mathtools}
\usepackage{fixmath}
\usepackage{siunitx}

\usepackage{pifont}
\newcommand{\cmark}{\ding{51}}
\newcommand{\xmark}{\ding{55}}

\usepackage{enumitem}
\setlist[enumerate]{itemsep=0.2ex, topsep=0.5\topsep}
\setlist[description]{itemsep=0.2ex, topsep=0.5\topsep}
\setlist[itemize]{itemsep=0.2ex, topsep=0.5\topsep}

\usepackage{setspace}
\usepackage[protrusion=true,expansion=false, activate={true,nocompatibility},final,kerning=true,spacing=true]{microtype}
\microtypecontext{spacing=nonfrench}
\usepackage{ellipsis}
\usepackage{xspace}
\usepackage{hfoldsty}

\usepackage{tcolorbox}
\usepackage[scaled=0.86]{helvet}
\usepackage{lmodern}

\usepackage[
    pdfa,
    hidelinks,
    pdftex, 
    pdfdisplaydoctitle,
    pdfpagelabels,
    pdfauthor={Christopher Morris},
    pdftitle={},
    pdfsubject={},
    pdfkeywords={},
    pdfproducer={Latex with the hyperref package},
    pdfcreator={pdflatex}
]{hyperref}

\makeatletter
\def\thmt@refnamewithcomma #1#2#3,#4,#5\@nil{%
	\@xa\def\csname\thmt@envname #1utorefname\endcsname{#3}%
	\ifcsname #2refname\endcsname
	\csname #2refname\expandafter\endcsname\expandafter{\thmt@envname}{#3}{#4}%
	\fi
}
\makeatother
\usepackage[capitalise,noabbrev]{cleveref}   

\newcommand{\new}[1]{\emph{#1}}

\newcommand{\cO}{\ensuremath{{\mathcal O}}\xspace}

\newcommand{\bbR}{\ensuremath{\mathbb{R}}}

\newcommand{\bbN}{\ensuremath{\mathbb{N}}}

\newcommand{\NN}{\mathbb{N}}

\newcommand{\wl}[2]{$(#1, #2)$\text{-}\textsf{LWL}\xspace}
\newcommand{\seq}[2]{$(#1, #2)$\text{-}\textsf{SpeqNet}\xspace}
\newcommand{\kwl}{$k$\text{-}\textsf{WL}\xspace}
\newcommand{\wlone}{$1$\text{-}\textsf{WL}\xspace}

\newcommand{\deltakwl}{$\delta$-$k$-\textsf{WL}\xspace}

\newcommand{\localkwl}{$\delta$-$k$-\textsf{LWL}\xspace}
\newcommand{\pluskwl}{$\delta$-$k$-\textsf{LWL}\xspace$\!\!^+$\xspace}
\newcommand{\pluswl}{$(k,s)$-\textsf{LWL}\xspace$\!\!^+$\xspace}

\newcommand{\kgnn}{$k$\textrm{-}\textsf{GNN}\xspace}
\newcommand{\kign}{$k$\textrm{-}\textsf{FGNN}\xspace}

\newcommand{\shp}{\textsf{SP}\xspace}
\newcommand{\gr}{\textsf{GR}\xspace}
\newcommand{\wloa}{\textsf{WLOA}\xspace}
\newcommand{\gin}{\textsf{GIN}\xspace}
\newcommand{\gine}{\textsf{GINE}\xspace}
\newcommand{\gineps}{\textsf{GIN-$\varepsilon$}\xspace}
\newcommand{\gineeps}{\textsf{GINE-$\varepsilon$}\xspace}

\newtheorem{claim}[theorem]{Claim}

\renewcommand{\vec}[1]{\mathbf{#1}}

\usepackage[auth-lg]{authblk}

\recalctypearea

\title{\LARGE\normalfont\bfseries SpeqNets: Sparsity-aware Permutation-equivariant Graph Networks}
\author[1,2]{Christopher Morris\footnote{Email: \texttt{chris@christophermorris.info}}}
\author[2]{Gaurav Rattan}
\author[3]{Sandra Kiefer}
\author[1]{Siamak Ravanbakhsh}

\affil[1]{McGill University and Mila Quebec AI Institute}
\affil[2]{RWTH Aachen University}
\affil[3]{Max Planck Institute for Software Systems}
\date{\vspace{-30pt}}

\begin{document}
\maketitle

\begin{abstract}
While message-passing graph neural networks have clear limitations in approximating permutation-equivariant functions over graphs or general relational data, more expressive, higher-order graph neural networks do not scale to large graphs. They either operate on $k$-order tensors or consider all $k$-node subgraphs, implying an exponential dependence on $k$ in memory requirements, and do not adapt to the sparsity of the graph. By introducing new heuristics for the graph isomorphism problem, we devise a class of universal, permutation-equivariant graph networks, which, unlike previous architectures, offer a fine-grained control between expressivity and scalability and adapt to the sparsity of the graph. These architectures lead to vastly reduced computation times compared to standard higher-order graph networks in the supervised node- and graph-level classification and regression regime while significantly improving standard graph neural network and graph kernel architectures in terms of predictive performance.
\end{abstract}

\section{Introduction}
Graph-structured data is ubiquitous across application domains ranging from chemo- and bioinformatics~\citep{Barabasi2004,Jum+2021,Sto+2020} to image~\citep{Sim+2017} and social-network analysis~\citep{Eas+2010}. To develop successful machine-learning models in these domains, we need techniques that exploit the rich information inherent in the graph structure and the feature information within nodes and edges. In recent years, numerous approaches have been proposed for machine learning with graphs---most notably, approaches based on \new{graph kernels}~\citep{Borg+2020,Kri+2019} or using \new{graph neural networks} (GNNs)~\citep{Cha+2020,Gil+2017,Gro+2021,Mor+2022}. Here, graph kernels based on the \new{$1$-dimensional Weisfeiler--Leman algorithm} (\wlone)~\citep{Wei+1968}, a simple heuristic for the graph isomorphism problem, and corresponding GNNs~\citep{Mor+2019,Xu+2018b} have recently advanced the state-of-the-art in supervised node- and graph-level learning. However, the \wlone operates via simple neighborhood aggregation, and the purely local nature of the related approaches misses important patterns in the given data. Moreover, they are only applicable to binary structures and therefore cannot deal with general structures containing relations of higher arity, e.g., hypergraphs. A more powerful algorithm for graph isomorphism testing is the \emph{$k$-dimensional Weisfeiler--Leman algorithm} (\kwl)~\citep{Bab1979,Cai+1992}.\footnote{In~\citep{Bab+2016}, László Babai mentions that he introduced the algorithm in 1979 together with Rudolf Mathon.}
The algorithm captures more global, higher-order patterns by iteratively computing a coloring or labeling for $k$-tuples defined over the set of nodes of a given graph based on a certain notion of adjacency between tuples. See \citep{Kie2020b} for a survey and more background. 
However, since the algorithm considers all $n^k$ many $k$-tuples of an $n$-node graph, it does not scale to large real-world graphs. Moreover, the cardinality of the considered neighborhood is always $k\cdot n$. Hence, a potential \emph{sparsity} of the input graph does not reduce the running time.

New neural architectures that possess the same power as the \kwl in terms of separating non-isomorphic graphs~\citep{Azi+2020,Gee+2020b,Mar+2019b} suffer from the same drawbacks, i.e., their memory requirement is lower-bounded by $n^k$ for an $n$-node graph, and they have to resort to dense matrix multiplication. Recently,~\citet{Morris2020b} introduced the local variant (\localkwl) of the \kwl considering only a subset of the neighborhoods in \kwl. However, like the original algorithm, the local variant operates on the set of all possible $k$-tuples, again resulting in the same (exponential) memory requirements, rendering the algorithm not practical for large, real-world graphs. 

\paragraph{Present work}
To address the described drawbacks, we introduce a new set of heuristics for the graph isomorphism problem, denoted \wl{k}{s}, which only considers a subset of all $k$-tuples, namely those \emph{inducing subgraphs with at most $s$ connected components}. We study the effect of $k$ and $s$ on the expressive power of the heuristics. Specifically, we show that the \wl{k}{1} induces a hierarchy of provably expressive heuristics for the graph isomorphism problem, i.e., with increasing $k$, the algorithm becomes strictly more expressive. Additionally, we prove that the \wl{k}{2} is strictly more expressive than the \wl{k}{1}. Further, we separate the \wl{k}{2} and \wl{k}{k} by showing that the \wl{k}{k} is strictly more expressive than the  \wl{k}{2}. Building on these combinatorial insights, we derive corresponding provably expressive, permutation-equivariant neural architectures, denoted $(k,s)$-\textsf{SpeqNets}, which offer a more fine-grained trade-off between scalability and expressivity compared to previous architectures based on the \kwl, see~\cref{overview} for a high-level overview of the theoretical results. Empirically, we show how our architectures offer vastly reduced computation times while beating baseline GNNs and other higher-order graph networks in terms of predictive performance on well-known node- and graph-level prediction benchmark datasets.

\begin{figure}[t]
	\begin{center}
		\resizebox{0.8\textwidth}{!}{
			
			\tikzset{
				treenode/.style = {shape=rectangle, rounded corners,
					draw, align=center,
					minimum width=50pt,
				}
			}
			\trimbox{0pt 20pt 0pt 10pt}{
				\begin{tikzpicture}[scale=1.4,font=\footnotesize,>=stealth', thick,sibling distance=15mm, level distance=30pt,minimum size=18pt, sibling distance=50pt]
						
					\node(aaa) at (-5.0,0) [treenode,fill=YellowGreen!60, minimum width=60pt] {$1$-\textsf{WL}};
					\node(a) at (-3.0,0) [treenode,fill=YellowGreen!60, minimum width=60pt] {\wl{1}{1}};
					\node(aa) at (-3.0,1) [treenode,fill=RubineRed!50, minimum width=60pt] {\seq{1}{1}};
					
					\node(aaaa) at (-5.0,1) [treenode,fill=RubineRed!50, minimum width=60pt] {\textsf{GNN}s};
					
					\node(b) at (-1.0,0) [treenode,fill=YellowGreen!60, minimum width=60pt] {\wl{2}{1}};
					\node(bb) at (-1.0,1) [treenode,fill=RubineRed!50, minimum width=60pt] {\seq{2}{1}};
					
					\node(c) at (1.0,0) [treenode,fill=YellowGreen!60, minimum width=60pt] {\wl{3}{1}};
					\node(cc) at (1.0,1) [treenode,fill=RubineRed!50, minimum width=60pt] {\seq{3}{1}};
					
					\node(d) at (4.0,0) [treenode,fill=YellowGreen!60, minimum width=60pt] {\wl{k}{1}};
					\node(dd) at (4.0,1) [treenode,fill=RubineRed!50, minimum width=60pt] {\seq{k}{1}};

					\node(f) at (4.0,-1.0) [treenode,fill=YellowGreen!60, minimum width=60pt] {\wl{k}{2}};
						    
					\node(g) at (1.0,-1.0) [treenode,fill=YellowGreen!60, minimum width=60pt] {\wl{k}{k}};
					
					\node(h) at (-1.0,-1.0) [treenode,fill=gray!30, minimum width=60pt] {\deltakwl};
					
					\node(i) at (-3.0,-1.0) [treenode,fill=gray!30, minimum width=60pt] {\pluskwl};
					
					\node(j) at (-5.0,-1.0) [treenode,fill=gray!30, minimum width=60pt] {\kwl};

					\node(e) at (2.5,0) [] {$\cdots$};
							
					\draw[->] (a) -- (b) node[midway,label={[shift={(0.0,-.39)}]\small $\sqsupset$}]{};
					\draw[->] (b) -- (c) node[midway,label={[shift={(0.0,-.39)}]\small $\sqsupset$}]{};
					\draw[->] (c) -- (e) node[midway,label={[shift={(0.0,-.39)}]\small $\sqsupset$}]{};
					\draw[->] (e) -- (d) node[midway,label={[shift={(0.0,-.39)}]\small $\sqsupset$}]{};
					
					\draw[->] (d) -- (f) node[midway,label={[shift={(-.25,-0.65)}]\small $\sqsupset$}]{};
							
					\draw[->] (f) -- (g) node[midway,label={[shift={(0.1,-0.39)}]\small $\sqsupset$}]{};
					
					\draw[<->] (g) -- (h) node[midway,label={[shift={(0.0,-0.39)}]\small $\equiv$}]{};
					
					\draw[<->] (a) -- (aa) node[midway,label={[shift={(-.25,-0.65)}]\small $\equiv$}]{};
					\draw[<->] (b) -- (bb) node[midway,label={[shift={(-.25,-0.65)}]\small $\equiv$}]{};
					\draw[<->] (c) -- (cc) node[midway,label={[shift={(-.25,-0.65)}]\small $\equiv$}]{};
					\draw[<->] (d) -- (dd) node[midway,label={[shift={(-.25,-0.65)}]\small $\equiv$}]{};
					
					\draw[<-] (i) -- (j) node[midway,label={[shift={(0.0,-.34)}]\small $\sqsupset^*$}]{};
					
					\draw[<-] (i) -- (h) node[midway,label={[shift={(0.0,-.34)}]\small $\sqsubset^*$}]{};	
					
					\draw[<->] (a) -- (aaa) node[midway,label={[shift={(0.0,-.39)}]\small $\equiv$}]{};
					
					\draw[<->] (aaa) -- (aaaa) node[midway,label={[shift={(-.25,-0.65)}]\small $\equiv$}]{};
			
					\draw[->, thick] (-5.7,1.5) -- (4.7,1.5) node [midway,fill=white] {\rotatebox{0}{\textbf{\footnotesize Approximate more functions}}};
				\end{tikzpicture}}}
	\end{center}
	\caption{Overview of the power of the proposed algorithms and neural architectures. The green and red nodes represent algorithms proposed in the present work. Forward arrows point to more powerful algorithms or neural architectures. $^*$---Proven in~\citep{Morris2020b}. $A \sqsubset B$ ($A \equiv B$): algorithm $A$ is strictly more powerful than (equally powerful as) $B$.}\label{overview}
\end{figure}
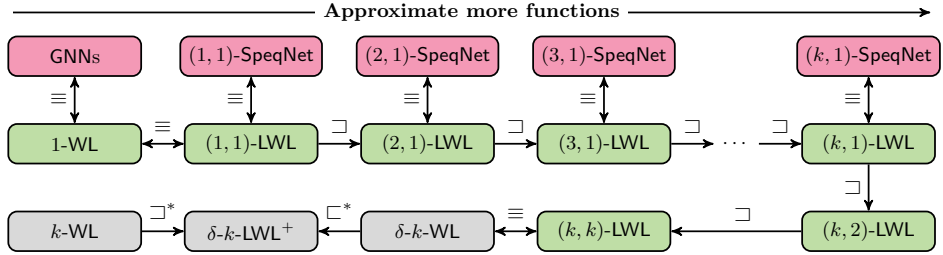

\subsection{Related work}

In the following, we review related work and background on graph kernels, GNNs, and graph theory.

\paragraph{Graph kernels}
Historically, kernel methods---which implicitly or explicitly map graphs to elements of a Hilbert space---have been the dominant approach for supervised learning on graphs. Important early work in this area includes random-walk based kernels~\citep{Gae+2003,Kas+2003,Kri+2017b} and kernels based on shortest paths~\citep{Bor+2005}. More recently, developments in the field have emphasized scalability, focusing on techniques that bypass expensive Gram matrix computations by using explicit feature maps, see, e.g.,~\citep{She+2011}. \citet{Mor+2017} devised a local, set-based variant of the \kwl and a corresponding kernel. However, the approach is weaker than the tuple-based algorithm. Further, \citet{Mor+2020} proposed kernels based on the \localkwl.

\citet{Yan+2015} successfully employed Graphlet~\citep{She+2009} and Weisfeiler--Leman kernels within frameworks for smoothed~\citep{Yan+2015} and deep graph kernels~\citep{Yan+2015a}. Other recent work focuses on assignment-based~\citep{Joh+2015,Kri+2016,Nik+2017}, spectral~\citep{Kon+2016,Ver+2017}, graph decomposition-based~\citep{Nik+2018}, randomized binning approaches~\citep{Hei+2019}, and the extension of kernels based on the \wlone~\citep{Mor+2016,Rie+2019,Tog+2019}. For a theoretical investigation of graph kernels, see~\citep{Kri+2018}, and for a thorough survey of graph kernels, see~\citep{Borg+2020,Kri+2019}.

\paragraph{GNNs}
Recently, GNNs~\citep{Gil+2017,Sca+2009} emerged as an alternative to graph kernels. Instances of this architecture are described in, e.g.,~\citep{Duv+2015,Ham+2017,Vel+2018}, which can be subsumed under the message-passing framework introduced in~\citep{Gil+2017}. In parallel, approaches based on spectral information were introduced in, e.g.,~\citep{Defferrard2016,Bru+2014,Kip+2017,Mon+2017}---all of which descend from early work in~\citep{Kir+1995,bas+1997,Spe+1997,mic+2005,Mer+2005,mic+2009,Sca+2009}.
Recent extensions and improvements to the GNN framework include approaches to incorporate different local structures (around subgraphs), e.g.,~\citep{Hai+2019,Fla+2020,Jin+2020,Nie+2016,Xu+2018}, novel techniques for pooling node representations in order to perform graph classification, e.g.,~\citep{Can+2018,Gao+2019,Yin+2018,Zha+2018}, incorporating distance information~\citep{You+2019}, and non-Euclidean geometry approaches~\citep{Cha+2019}. Moreover, recently, empirical studies on neighborhood-aggregation functions for continuous node features~\citep{Cor+2020}, edge-based GNNs leveraging physical knowledge~\citep{And+2019,Kli+2020,Kli+2021}, and sparsification methods~\citep{Ron+2020} emerged. Surveys of recent advancements in GNN techniques can be found, e.g., in~\citep{Cha+2020,Wu+2019,Zho+2018}. 

\paragraph{Limits of GNNs and more expressive architectures} 
Recently, connections of GNNs to Weisfeiler--Leman type algorithms have been shown~\citep{,Azi+2020,Bar+2020,Che+2019,Gee+2020a,Gee+2020b,Mae+2019,Mar+2019,Mor+2019,Xu+2018b}. Specifically,~\citep{Mor+2019,Xu+2018b} showed that the expressive power of any possible GNN architecture is limited by the \wlone in terms of distinguishing non-isomorphic graphs. 

Triggered by the above results, a large set of papers proposed architectures to overcome the expressivity limitations of the \wlone.
\citet{Mor+2019} introduced \emph{$k$-dimensional} GNNs (\kgnn) which rely on a message-passing scheme between subgraphs of cardinality~$k$. Similarly to~\citep{Mor+2017}, the paper employed a local, set-based (neural) variant of the \kwl.
Later, this was refined in~\citep{Mar+2019,Azi+2020} by introducing \emph{$k$-order folklore graph neural networks} (\kign), which are equivalent to the folklore or oblivious variant of the \kwl~\citep{Gro+2021,Mor+2022} in terms of distinguishing non-isomorphic graphs. Subsequently, \citet{Morris2020b} introduced neural architectures based on the \localkwl, which only considers a subset of the neighborhood from the \kwl, taking the potential sparsity of the underlying graph (to some extent) into account. 
Although more scalable, the algorithm reaches computational exhaustion on large graphs since it considers all $n^k$ tuples of size $k$. \citet{Che+2019} connected the theory of universal approximations of permutation-invariant functions and the graph isomorphism viewpoint and introduced a variation of the $2$-\textsf{WL}. See~\citep{Mor+2022} for an in-depth survey on this topic. 

Recent work has extended the expressive power of GNNs, e.g., by encoding node identifiers~\citep{Mur+2019b, Vig+2020}, leveraging random features~\citep{Abb+2020,Das+2020,Sat+2020}, subgraph information~\citep{Bev+2021,botsas2020improving,Cot+2021,Pap+2021,Thi+2021,you2021identity,Zha+2021,Zha+2021b}, homomorphism counts~\citep{Bar+2021,Hoa+2020}, spectral information~\citep{Bal+2021}, simplicial and cellular complexes~\citep{Bod+2021,Bod+2021b}, random walks~\citep{Toe+2021}, graph decompositions~\citep{Tal+2021}, distance~\citep{li2020distance} and directional information~\citep{beaini2020directional}.

However, all of the above approaches mentioned in the previous paragraph only overcome limitations of the \wlone, $2$-\textsf{WL}, or $3$-\textsf{WL}, and do not induce a hierarchy of provably powerful, permutation-equivariant neural architectures aligned with the \kwl hierarchy. 

\paragraph{Theory}
The Weisfeiler--Leman algorithm constitutes one of the earliest and most natural approaches to isomorphism testing~\citep{Wei+1976,Wei+1968}, having been heavily investigated by the theory community over the last few decades~\citep{Gro2017}. Moreover, the fundamental nature of the \kwl is evident from a variety of connections to other fields such as logic, optimization, counting complexity, and quantum computing. The power and limitations of the \kwl can be neatly characterized in terms of logic and descriptive complexity~\citep{Bab1979,Imm+1990}, Sherali-Adams relaxations of the natural integer linear program for the graph isomorphism problem~\citep{Ast+2013,GroheO15,Mal2014}, homomorphism counts~\citep{Del+2018}, and quantum isomorphism games~\citep{Ats+2019}. In their seminal paper,~\citet{Cai+1992} showed that, for each $k$, there exists a pair of non-isomorphic graphs of size $\cO(k)$ that are not distinguished by the \kwl. \citep{Kie2020a} gives a thorough survey of more background and related results concerning the expressive power of the \kwl. For $k=1$, the power of the algorithm has been completely characterized~\citep{Arv+2015,Kie+2015}.  Moreover, upper bounds on the running time~\citep{Ber+2017a} and the number of iterations for $k=1$~\citep{Kie+2020} and for the folklore $k=2$~\citep{Kie+2016,Lic+2019} have been shown. For $k$ in $\{1,2\}$,~\citet{Arv+2019} studied the abilities of the (folklore) \kwl to detect and count fixed subgraphs, extending the work of~\citet{Fue+2017}. The former was refined in~\citep{Che+2020a}. \citet{Kie+2019} showed that the folklore $3$-$\mathsf{WL}$ completely captures the structure of planar graphs. The algorithm (for logarithmic $k$) plays a prominent role in the recent result of \cite{Bab+2016} improving the best known running time for the graph isomorphism problem. Recently,~\citet{Gro+2020a} introduced the framework of Deep Weisfeiler--Leman algorithms, which allow the design of a more powerful graph isomorphism test than Weisfeiler--Leman type algorithms. Finally, the emerging connections between the Weisfeiler--Leman paradigm and graph learning are described in two recent surveys~\citep{Gro+2020,Mor+2022}.

\section{Preliminaries}\label{prelim_ext}

As usual, for $n \geq 1$, let $[n] \coloneqq \{ 1, \dotsc, n \} \subset \NN$. We use $\{\!\!\{ \dots\}\!\!\}$ to denote multisets, i.e., the generalization of sets allowing for multiple instances for each of its elements.

\paragraph{Graphs} A \new{graph} $G$ is a pair $(V(G),E(G))$ with \emph{finite} sets of
\new{nodes} $V(G)$ and \new{edges} $E(G) \subseteq \{ \{u,v\}
\subseteq V(G) \mid u \neq v \}$. If not otherwise stated, we set $n \coloneqq |V(G)|$. For ease of
notation, we denote the edge $\{u,v\}$ in $E(G)$ by $(u,v)$ or
$(v,u)$. In the case of \emph{directed graphs}, $E \subseteq \{ (u,v)
\in V \times V \mid u \neq v \}$. A \new{labeled graph} $G$ is a triple
$(V,E,\ell)$ with a label function $\ell \colon V(G) \cup E(G) \to \bbN$. Then $\ell(v)$ is a
\new{label} of $v$ for $v$ in $V(G) \cup E(G)$. The \new{neighborhood} 
of $v$ in $V(G)$ is denoted by $\delta(v) = \{ u \in V(G) \mid \{ v, u \} \in E(G) \}$ and the \new{degree} of a node $v$ is  $|\delta(v)|$. For $S \subseteq
V(G)$, the graph $G[S] = (S,E_S)$ is the \new{subgraph induced by $S$}, where
$E_S = \{ (u,v) \in E(G) \mid u,v \in S \}$. 
A \new{connected component} of a graph $G$ is an inclusion-wise maximal subgraph of $G$ in which every two nodes are connected by paths. A \new{tree} is a connected graph without
cycles. A \new{rooted tree} is an oriented tree with a designated node called \new{root}, in which the edges are directed away from the root. 
Let $p$ be a node in a rooted tree. Then we call its out-neighbors \new{children} with \new{parent} $p$. 
We denote an undirected \new{cycle} on $k$ nodes by $C_k$.  
Given two graphs $G$ and $H$ with disjoint node sets, we denote their disjoint union by $G \,\dot\cup\, H$.

Two graphs $G$ and $H$ are \new{isomorphic} and we write $G \simeq H$ if there exists a bijection $\varphi \colon V(G) \to V(H)$ preserving the adjacency relation, i.e., $(u,v)$ is in $E(G)$ if and only if
$(\varphi(u),\varphi(v))$ is in $E(H)$. Then $\varphi$ is an \new{isomorphism} between
$G$ and $H$. Moreover, we call the equivalence classes induced by
$\simeq$ \emph{isomorphism types}, and denote the isomorphism type of $G$ by
$\tau_G$. In the case of labeled graphs, we additionally require that
$\ell(v) = \ell(\varphi(v))$ for $v$ in $V(G)$ and $\ell((u,v)) = \ell((\varphi(u), \varphi(v)))$ for $(u,v)$ in $E(G)$. 
Let $\vec{v}$ be a \emph{tuple} in $V(G)^k$ for $k > 0$, then $G[\vec{v}]$ is the subgraph induced by the multiset of elements of $\vec{v}$, where the nodes are labeled with integers from $\{ 1, \dots, k \}$ corresponding to their positions in $\vec{v}$.

\paragraph{Equivariance} For $n > 0$, let $S_n$ denote the set of all permutations of $[n]$, i.e., the set of all bijections from $[n]$ to itself. For $\sigma$ in $S_n$ and a graph $G$ such that $V(G) = [n]$, let $\sigma \cdot G$ be the graph such that $V(\sigma \cdot G) = \{ v_{\sigma^{-1}(1)}, \dots, v_{\sigma^{-1}(n)} \}$ and $E(\sigma \cdot G) = \{ (v_{\sigma^{-1}(i)},v_{\sigma^{-1}(j)}) \mid (v_i, v_j) \in E(G)  \}$. That is, applying the permutation $\sigma$ reorders the nodes. Hence, for two isomorphic graphs $G$ and $H$ on the same vertex set, i.e., $G \simeq H$, there exists $\sigma$ in $S_n$ such that $\sigma \cdot G = H$.   
Let $\mathcal{G}$ denote the set of all graphs, and let $\mathcal{G}_n$ denote the set of all graphs on $n$ nodes. A function $f\colon \mathcal{G} \rightarrow \bbR$ is \new{invariant} if for every $n>0$ and every $\sigma$ in $S_n$ and every graph $G$ of order $n$, it holds that $f(\sigma \cdot G) = f(G)$. A function $f \colon\mathcal{G} \to \mathcal{G}$ is \new{equivariant} if for every $n>0$, $f(\mathcal{G}_n) \subseteq \mathcal{G}_n$
and for every $\sigma$ in $S_n$, $f(\sigma \cdot G) = \sigma \cdot f(G)$. 

\paragraph{Kernels} A \emph{kernel} on a non-empty set $\mathcal{X}$ is a symmetric, positive semidefinite function 
$k \colon \mathcal{X} \times \mathcal{X} \to \mathbb{R}$.
Equivalently, a function $k\colon \mathcal{X} \times \mathcal{X} \to \mathbb{R}$ is a kernel if there is a \emph{feature map}
$\phi \colon \mathcal{X} \to \mathcal{H}$ to a Hilbert space $\mathcal{H}$ with inner product 
$\langle \cdot, \cdot \rangle$ such that 
$k(x,y) = \langle \phi(x),\phi(y) \rangle$ for all $x$ and $y$ in $\mathcal{X}$. A \emph{graph kernel} is a kernel on the set $\mathcal{G}$ of all graphs. 

\subsection{Node-refinement algorithms}\label{vr_ext}

In the following, we briefly describe the Weisfeiler--Leman algorithm and related variants~\citep{Morris2020b}. Let $k$ be a fixed positive integer. There exist two definitions of the \kwl, the so-called oblivious \kwl and the folklore or non-oblivious \kwl, see, e.g.,~\citep{Gro+2021}. There is a subtle difference in how they aggregate neighborhood information. Within the graph learning community, it is customary to abbreviate the oblivious \kwl as \kwl, a convention that we follow in this paper. 

We proceed to the definition of the \kwl. Let $V(G)^k$ denote the set of $k$-tuples of nodes of the graph $G$. A \new{coloring} of $V(G)^k$ is a mapping $C \colon V(G)^k \to \mathbb{N}$, i.e., we assign a number (color) to every tuple in $V(G)^k$. The \new{initial coloring} $C_{0}$ of $V(G)^k$ is specified by the atomic types of the tuples. So two tuples $\vec{v}$ and $\vec{w}$ in $V(G)^k$ have the same initial color iff the mapping $v_i \mapsto w_i$ induces an isomorphism between the labeled subgraphs $G[\vec{v}]$ and $G[\vec{w}]$. Note that, given a tuple $\vec{v}$ in $V(G)^k$, we can upper-bound the running time of the computation of this initial coloring for $\vec{v}$ by $\cO(k^2)$. A \new{color class} corresponding to a color $c$ is the set of all tuples colored $c$, i.e., the set $C^{-1}(c)$.

For $j$ in $[k]$ and $w$ in $V(G)$, let $\phi_j(\vec{v},w)$ be the $k$-tuple obtained by replacing the 
$j{\text{th}}$ component of $\vec{v}$ with the node $w$. That is, $\phi_j(\vec{v},w) = (v_1, \dots, v_{j-1}, w, v_{j+1}, \dots, v_k)$. If $\vec{w} = \phi_j(\vec{v},w)$ for some $w$ in $V(G)$, call $\vec{w}$ a $j$-\new{neighbor} of $\vec{v}$. The \new{neighborhood} of $\vec{v}$ is  the set of all $\vec{w}$ such that $\vec{w} = \phi_j(\vec{v},w)$ for some $j$ in $[k]$ and a $w \in V(G)$. 

The \new{refinement} of a coloring $C \colon V(G)^k \to \mathbb{N}$, denoted by $\widehat{C}$, is the coloring $\widehat{C} \colon V(G)^k \to \mathbb{N}$ defined as follows. 
For each $j$ in $[k]$, collect the colors of the $j$-neighbors of $\vec{v}$ in a multiset $S_j \coloneqq \{\!\! \{  C(\phi_j(\vec{v},w)) \mid w \in V(G) \} \!\!\}$.
Then, for a tuple $\vec{v}$, define
\[
	\widehat{C}(\vec{v}) \coloneqq (C(\vec{v}), M(\vec{v})),
\]
where $M(\vec{v})$ is the $k$-tuple $(S_1,\dots,S_k)$. For consistency, the obtained strings $\widehat{C}(\vec{v})$ are lexicographically sorted and renamed as integers, not used in previous iterations. Observe that the new color $\widehat{C}(\vec{v})$ is solely dictated by the color histogram of the neighborhood of $\vec{v}$. In general, a different mapping $M(\cdot)$ could be used, depending on the neighborhood information that we would like to aggregate. We will refer to a mapping $M(\cdot)$ as an \new{aggregation map}. 

\paragraph{$\boldsymbol{k}$-dimensional Weisfeiler--Leman}\label{wl_app} For $k\geq 2$, the \kwl computes a coloring $C_\infty \colon V(G)^k \to \mathbb{N}$ of a given graph $G$, as described next.\footnote{We define the $1$-\textsf{WL} in the next subsection.} To begin with, the initial coloring $C_0$ is computed. Then, starting with $C_0$, successive refinements $C_{i+1} = \widehat{C_i}$ are computed until convergence. More precisely,
\[
	C_{i+1}(\vec{v}) \coloneqq (C_i(\vec{v}), M_i(\vec{v})),
\]
where 
\begin{equation}\label{app:mi_ext}
	M_i(\vec{v}) =   \big( \{\!\! \{  C_i(\phi_1(\vec{v},w)) \mid w \in V(G) \} \!\!\}, \dots, \{\!\! \{  C_i(\phi_k(\vec{v},w)) \mid w \in V(G) \} \!\!\} \big).
\end{equation}
The successive refinement steps are also called \new{rounds} or \new{iterations}. Since the color classes form a partition of $V(G)^k$, there must exist a finite $\ell \leq |V(G)|^k$ such that $C_{\ell}$ and $\widehat{C_{\ell}}$ induce the same partition on the vertex tuples, i.e., the partition induced by $C_\ell$ cannot be refined further. The \kwl outputs this $C_\ell$ as the \emph{stable coloring} $C_\infty$. 

The \kwl \new{distinguishes} two graphs $G$ and $H$ if, upon running the \kwl on their disjoint union $G \,\dot\cup\, H$, there exists a color $c$ in $\mathbb{N}$ in the stable coloring such that the corresponding color class $S_c$ satisfies
\begin{equation*}
	|V(G)^k \cap S_c| \neq |V(H)^k \cap S_c|,
\end{equation*}
i.e., the numbers of $c$-colored tuples in $V(G)^k$ and $V(H)^k$ differ. Two graphs that are distinguished by the \kwl must be non-isomorphic, because the algorithm is defined in an isomorphism-invariant way. 

Finally, the application of different aggregation maps $M$ yields related versions of \kwl. For example, setting $M(\cdot)$ to be 
\begin{equation*}
	M^F(\vec{v}) =  \{\!\! \{ \big( C(\phi_1(\vec{v},w)) , \dots,   C(\phi_k(\vec{v},w)) \big)  \mid w \in V(G) \} \!\!\},
\end{equation*}
yields the so-called folklore version of \kwl (see e.g.,~\citep{Cai+1992}). It is known that the oblivious version of the \kwl is as powerful as the folklore version of the $(k\!-\!1)$-\textsf{WL}~\citep{Gro+2021}.

\paragraph{Local $\boldsymbol{\delta}$-$\boldsymbol{k}$-dimensional Weisfeiler--Leman algorithm}\label{lwl_ext}

\citet{Morris2020b} introduced a more efficient variant of the \kwl, the \new{local $\delta$-$k$-dimensional Weisfeiler--Leman algorithm} (\localkwl). In contrast to the \kwl, the  \localkwl considers only a subset of the entire neighborhood of a node tuple. Let the tuple $\vec{w} = \phi_j(\vec{v},w)$ be a $j$-{neighbor} of $\vec{v}$. We say that $\vec{w}$ is a \new{local} $j$-neighbor of $\vec{v}$ if $w$ is adjacent to the replaced node $v_j$. Otherwise, the tuple $\vec{w}$ is a \new{global} $j$-neighbor of $\vec{v}$. The \localkwl considers only local neighbors during the neighborhood aggregation process, and discards any information about the global neighbors. Formally, the \localkwl refines a coloring $C^{k,\delta}_i$ (obtained after $i$ rounds of the \localkwl) via the aggregation function 
\begin{equation}\label{eqnmidd_ext}
	\begin{split}
		M^{\delta}_i(\vec{v}) =   \big( \{\!\! \{ C^{k, \delta}_{i}(\phi_1(\vec{v},w)) \mid w \in \delta(v_1) \} \!\!\}, \dots, \{\!\! \{  C^{k, \delta}_{i}(\phi_k(\vec{v},w)) \mid w \in \delta(v_k) \}  \!\!\} \big),
	\end{split}
\end{equation}		
hence considering only the local $j$-neighbors of the tuple $\vec{v}$ in each iteration.  The coloring functions for the iterations of the \localkwl are then defined by 
\begin{equation}\label{wlsimple_ext}
	C^{k,\delta}_{i+1}(\vec{v}) = (C^{k,\delta}_{i}(\vec{v}), M^{\delta}_i(\vec{v})).
\end{equation}
We define the $1$-\textsf{WL} to be the $\delta$-1-\textsf{LWL}, which is commonly known as Color Refinement or Naive Node Classification.\footnote{Strictly speaking, the \wlone and Color Refinement are two different algorithms. That is, the \wlone considers neighbors and non-neighbors to update the coloring, resulting in a slightly higher expressivity when distinguishing nodes in a given graph, see~\citep{Gro+2021} for details.} Hence, we can equivalently define 
\begin{equation}\label{ck_ext}
	C^{1,\delta}_{i+1}(v) \coloneqq \left(C^{1,\delta}_{i}(v), \{\!\!\{ C^{1,\delta}_{i}(w) \mid w \in \delta(v) \}\!\!\}\right),
\end{equation}
for a node $v$ in $V(G)$. 

\citet{Morris2020b} also defined the \pluskwl, a minor variation of the \localkwl.  Formally, the \pluskwl refines a coloring $C_i$ (obtained after $i$ rounds) via the aggregation function
\begin{equation}\label{middp_ext}
	\begin{split}
		M^{\delta,+}(\vec{v}) =   \big( &\{\!\! \{ (C^{k, \delta}_i(\phi_1(\vec{v},w)), \#_{i}^1(\vec{v},\phi_1(\vec{v},w))) \hspace{0.4pt}\mid w \in \delta(v_1) \} \!\!\}, \dots, \\ &\{\!\! \{  (C^{k, \delta}_i(\phi_k(\vec{v},w)), \#_{i}^k(\vec{v},\phi_k(\vec{v},w))) \mid w \in \delta(v_k) \}  \!\!\} \big),
	\end{split}
\end{equation}
instead of the \localkwl aggregation defined in~\cref{eqnmidd_ext}. 
Here, we set
\begin{equation}\label{sharp_ext}
	\#_{i}^j(\vec{v}, \vec{x}) \coloneqq \big|\{ \vec{w} \colon  \vec{w} \sim_j \vec{v}, \, C^{k, \delta}_{i}(\vec{w}) = C^{k, \delta}_{i}(\vec{x})   \} \big|,
\end{equation}
where $\vec{w} \sim_j \vec{v}$ denotes that $\vec{w}$ is a $j$-neighbor of $\vec{v}$, for $j$ in $[k]$. Essentially, $\#_{i}^j(\vec{v}, \vec{x})$ counts the number of (local or global) $j$-neighbors of $\vec{v}$ which have the same color as $\vec{x}$ under the coloring $C_i$ (i.e., after $i$ rounds). \citet{Morris2020b} showed that the \pluskwl is slightly more powerful than the \kwl in distinguishing non-isomorphic graphs. 

\paragraph{The Weisfeiler--Leman hierarchy and permutation-invariant function approximation}\label{connect}
The Weisfeiler--Leman hierarchy is a purely combinatorial algorithm for testing graph isomorphism. However,  the graph isomorphism function, mapping non-isomorphic graphs to different values, is the hardest to approximate permutation-invariant function. Hence, the Weisfeiler--Leman hierarchy has strong ties to GNNs' capabilities to approximate permutation-invariant or equivariant functions over graphs. For example,~\citet{Mor+2019,Xu+2018b} showed that the expressive power of any possible GNN architecture is limited by \wlone in terms of distinguishing non-isomorphic graphs. \citet{Azi+2020} refined these results by showing that if an architecture is capable of simulating \kwl and allows the application of universal neural networks on vertex features, it will be able to approximate any permutation-equivariant function below the expressive power of \kwl; see also~\citep{Che+2019}. Hence, if one shows that one architecture distinguishes more graphs than another, it follows that the corresponding GNN can approximate more functions. These results were refined in \citep{geerts2022} for Color Refinement and taking into account the number of iterations of \kwl.

\section{The \texorpdfstring{\wl{k}{s}}{(k,s)-LWL} algorithm}
Since both \kwl and its local variant \localkwl consider all $k$-tuples of a graph, they do not scale to large graphs for larger $k$. Specifically, for an $n$-node graph, the memory requirement is in $\Omega(n^k)$. Further, since the \kwl considers the graph structure only at initialization, it does not adapt to its sparsity, i.e., it does not run faster on sparser graphs. To address this issue, we introduce the \wl{k}{s}. The algorithm offers more fine-grained control over the trade-off between expressivity and scalability by only considering a subset of all $k$-tuples, namely those inducing subgraphs with at most $s$ \emph{connected components}. This combinatorial algorithm will be the basis of the permutation-equivariant neural architectures of~\cref{gn}.

Let $G$ be a graph. Then $\mathsf{\#com}(G)$ denotes the number of (connected) components of $G$. Further, let $k \geq 1$ and $1 \leq s \leq k$, then
\begin{equation*}
	V(G)^k_s \coloneqq \{ \vec{v} \in V(G)^k \mid \mathsf{\#com}(G[\vec{v}]) \leq s  \}
\end{equation*}
is the set of \new{$(k,s)$-tuples} of nodes, i.e, $k$-tuples which induce (sub-)graphs with at most $s$ (connected) components.

In contrast to the algorithms of~\cref{vr_ext}, the \wl{k}{s} colors tuples from  $V(G)^k_s$ instead of the entire $V(G)^k$. Hence, analogously to~\cref{vr_ext}, a coloring of $V(G)^k_s$ is a mapping $C^{k,s} \colon V(G)^k_s \to \mathbb{N}$, assigning a number (color) to every tuple in $V(G)^k_s$. The initial coloring $C^{k,s}_{0}$ of $V(G)^k_s$ is defined in the same way as before, i.e., specified by the isomorphism types of the tuples, see~\cref{vr_ext}.
Subsequently, the coloring is updated using the \localkwl aggregation map, see~\cref{eqnmidd_ext}. Hence, the  \wl{k}{s} is a variant of the \localkwl considering only $(k,s)$-tuples, i.e.,~\cref{eqnmidd_ext} is replaced with 
\begin{equation}\label{eqnmiddd}
	\begin{split}
		M^{\delta,k,s}_i(\vec{v}) \coloneqq  \big( &\{\!\! \{ C^{k,s}_{i}(\phi_1(\vec{v},w)) \mid w \in \delta(v_1) \text{  and } \phi_1(\vec{v},w) \in V(G)^k_s  \} \!\!\},\dots, \\ 
		&\{\!\! \{  C^{k,s}_{i}(\phi_k(\vec{v},w)) \mid w \in \delta(v_k) \text{ and }  \phi_k(\vec{v},w) \in V(G)^k_s  \}  \!\!\} \big),
	\end{split}
\end{equation}		
i.e., $M^{\delta}_i(\vec{v})$ restricted to colors of $(k,s)$-tuples. The stable coloring $C^{k,s}_\infty$ is defined analogously to the stable coloring  $C^{k}_\infty$. In the following two subsections, we investigate the properties of the algorithm in detail.

Analogously to the \pluskwl, we also define the \pluswl using 
\begin{equation*}\label{middp2}
	\begin{split}
		M^{\delta,+}(\vec{v}) =   \big( &\{\!\! \{ (C^{k,s}_i(\phi_1(\vec{v},w)), \#_{i,s}^1(\vec{v},\phi_1(\vec{v},w))) \hspace{0.4pt}\mid w \in \delta(v_1) \text{  and } \phi_1(\vec{v},w) \in V(G)^k_s \} \!\!\}, \dots, \\ &\{\!\! \{  (C^{k,s}_i(\phi_k(\vec{v},w)), \#_{i,s}^k(\vec{v},\phi_k(\vec{v},w))) \mid w \in \delta(v_k) \text{  and } \phi_k(\vec{v},w) \in V(G)^k_s \}  \!\!\} \big),
	\end{split}
\end{equation*}
where the function
\begin{equation*}\label{sharp2}
	\#_{i,s}^j(\vec{v}, \vec{x}) = \big|\{ \vec{w} \colon \vec{w} \sim_j \vec{v}, \, C^{k,s}_{i}(\vec{w}) = C^{k,s}_{i}(\vec{x}) \text{ and }  \vec{w} \in  V(G)^k_s \} \big|,
\end{equation*}
restricts $\#_{i}^j(\vec{v}, \vec{x})$ to $(k,s)$-tuples.

\subsection{Expressivity}

Here, we investigate the expressivity of the \wl{k}{s}, i.e., its ability to distinguish non-isomorphic graphs, for different choices of $k$ and $s$. In~\cref{gn}, we will leverage these results to devise universal, permutation-equivariant graph networks. We start off with the following simple observation. Since the \wl{k}{k} colors all $k$-tuples, it is equal to the \localkwl.
\begin{observation}
	Let $k \geq 1$, then 
	\begin{equation*}
		\text{\wl{k}{k}} \equiv \text{\localkwl}\quad\text{and} \quad \text{\wl{1}{1}} \equiv \text{$\delta$-$1$-\textsf{LWL}\xspace} \equiv \text{\wlone}.
	\end{equation*}
\end{observation}

The following result shows that the \wl{k}{1} form a \emph{hierarchy}, i.e., the algorithm becomes more expressive as $k$ increases.
\begin{theorem}\label{theorem:one}
	Let $k \geq 1$, then 
	\begin{equation*}
		\text{\wl{k+1}{1}} \sqsubset \text{\wl{k}{1}}.
	\end{equation*}
\end{theorem}
Moreover, we also show that the \wl{k}{2} is more expressive than the \wl{k}{1}.
\begin{proposition}\label{theorem:two}
	For $k \geq 2$, it holds that 
	\begin{equation*}
		\text{\wl{k}{2}} \sqsubset \text{\wl{k}{1}}.
	\end{equation*}
\end{proposition}
Further, the following theorem yields that increasing the parameter $s$ results in higher expressivity. Formally, we show that the \wl{k}{k} is strictly more expressive than the \wl{k}{2}.
\begin{theorem}\label{theorem:three}
	For $k \geq 2$, it holds that 
	\begin{equation*}
		\text{\wl{k}{k}} \sqsubset \text{\wl{k}{2}}.
	\end{equation*}
\end{theorem}
See~\cref{time_ext} for an analysis of the asymptotic running time of the \wl{k}{s}, showing that it only depends on $k$, $s$, and the sparsity of the graph. In particular, the running time of the \wl{k}{s} on an $n$-vertex graph of bounded degree is $\tilde{\cO}(n^s)$ instead of the usual $\tilde{\cO}(n^k)$ for the \kwl, for fixed $k$ and $s$.

\subsubsection{Proofs of~\cref{theorem:one,theorem:two,theorem:three}}

To prove~\cref{theorem:one}, we introduce the \new{$(k,s)$-tuple graph}. It essentially contains all $(k,s)$-tuples as nodes, where each node $v_\vec{t}$ is labeled by the isomorphism type of the $(k,s)$-tuple $\vec{t}$. We join two nodes by an edge, labeled $j$, if the underlying $(k,s)$-tuples are $j$-neighbors. The formal definition of the $(k,s)$-tuple graph is as follows. Recall that $\tau$ denotes an isomorphism type.
\begin{definition}\label{def:ktuplegraph}
	Let $G$ be a graph and let $k \geq 1$, and $s$ in $[k]$. Further, let $\vec{s}$ and $\vec{t}$ be tuples in $V(G)_s^k$. Then the
	directed, labeled \new{$(k,s)$-tuple graph} $T^k_s(G) = (V_T, E_T, \ell_T)$ has node set $V_T = \{ v_\vec{t} \mid \vec{t} \in V(G)^k_s \}$, and
	\begin{equation}\label{ktuple}
		(v_\vec{s},v_\vec{t}) \in E_T \iff \vec{t} = \phi_j(\vec{s},w) \ \text{holds for some $j$ in $[k]$ and some $w$ in $V(G)$.}
	\end{equation} 
	We set $\ell_T((v_\vec{s}, v_\vec{t}))
	\coloneqq j$ if $\vec{t}$ is a local $j$-neighbor of $\vec{s}$, and let $\ell_T(v_\vec{s}) \coloneqq \tau_{G[\vec{s}]}$.
\end{definition}
Given a graph $G$ and the corresponding $(k,s)$-tuple graph $T^k_s(G)$, we define a variant of the \wlone, which takes into account edge labels. Namely, for $v_\vec{t}$ in $V_T$, the new algorithm uses the colorings $C^{1,\delta,*}_{0}(v_\vec{t}) = \tau_{G[\vec{t}]}$ and
\begin{equation}\label{wlv}
	C^{1,\delta,*}_{i+1}(v_\vec{t}) = (C^{1,\delta,*}_{i}(v_\vec{t}), \{\!\!\{ (C^{1,\delta,*}_{i}(v_\vec{s}), \ell(v_\vec{t}, v_\vec{s})) \mid v_\vec{s} \in \delta(v_\vec{t}) \} \!\!\}
\end{equation}
for $i >0$. Note that the \wlone, see~\cref{ck_ext}, and the variant defined via~\cref{wlv} have the same asymptotic running time. The following lemma states that the \wl{k}{s} can be simulated on the $(k,s)$-tuple graph using the above variant of the $1$-WL.
\begin{lemma}
	\label{lemma:wlk}
	Let $G$ be a graph, $k \geq 1$, and $s$ in $[k]$. Then
	\begin{equation*}
		C^{k,s}_{i}(\vec{t}) = C^{k,s}_{i}(\vec{u}) \iff C^{1,\delta,*}_{i}(v_\vec{t}) = C^{1,\delta,*}_{i}(v_\vec{u}),
	\end{equation*}
	for all $i \geq 0$, and all $(k,s)$-tuples $\vec{t}$ and $\vec{u}$ in $V(G)^k_s$.
\end{lemma}
\begin{proof}[Proof sketch]
	Induction on the number of iterations using~\cref{def:ktuplegraph}. 
\end{proof}

The \new{unrolling} of a neighborhood around a node of a given graph to a tree is defined as follows, see~\cref{unroll} for an illustration.	
\begin{figure}[t]
	\begin{center}
		\subfloat[]{
			\centering
			\input{figures/v.tex}}\qquad\qquad
		\subfloat[]{
			\centering
			\input{figures/u.tex}}
	\end{center}
	\caption{Illustration of the unrolling operation around the node $a$ for $i=2$.}\label{unroll}
\end{figure}
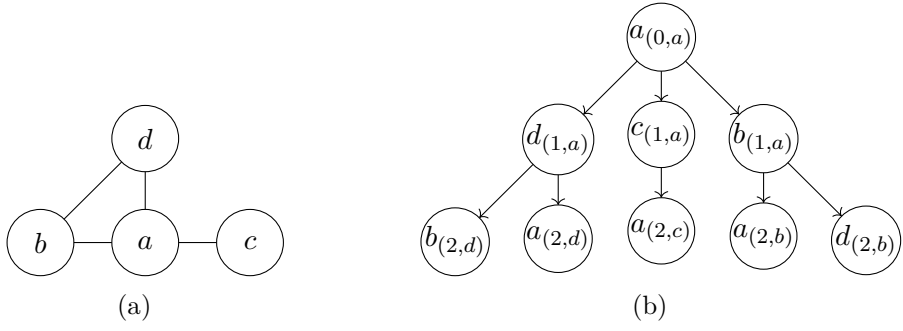
\begin{definition}
	Let $G=(V,E,\ell)$ be a labeled (directed) graph and let $v$ be in
	$V$. Then $U^i_{G,v} = (W_i, F_i, l_i)$ for $i \geq 0$ denotes the \new{unrolled tree} $G$
	around $v$ at depth $i$, where
	\begin{align*}                                                                                                                                                                                & W_i =                 
		\begin{cases} 
		\{v_{(0,v)}\}                                                                                                                                                                                 & \text{\;\; if } i = 0  \\ 
		\mathrlap{W_{i-1} \cup \{u_{(i,w_{(i-1,p)})} \mid u \in \delta(w) \text{ for } w_{(i-1,p)} \in W_{i-1} \}} \phantom{ F_{i-1} \cup \{ (w_{(i,p)},u_{(i,w)}) \mid u \in \delta(w) \text{ for } w_{(i-1,p)} \in W_{i-1} \}} & \text{\;\; otherwise,} 
		\end{cases}\\
		\text{and} &   \\                                                                                                                                                                        & F_i \;=               
		\begin{cases} 
		\emptyset                                                                                                                                                                             & \text{if } i = 0  \\ 
		F_{i-1} \cup \{ (w_{(i-1,p)},u_{(i,w)}) \mid u \in \delta(w)   \text{ for } w_{(i-1,p)} \in W_{i-1} \}                                                                                                   & \text{otherwise}. 
		\end{cases} 
	\end{align*}
	The label function is defined as $l_i(u_{(j,p)}) = \ell(u)$ for $u$ in $V$, and $l_i(u_{(j,w)}) = \ell((w,u))$. For notational convenience, we usually omit the subscript $i$.
\end{definition}	
In the following, we use the unrolled tree for the above defined $(k,s)$-tuple graph. 
For $k\ge 2$ and $s$ in $[k]$, we denote the \emph{directed}, unrolled tree of the
$(k,s)$-tuple graph of $G$ around the node $v_\vec{t}$ at depth $i$ for the tuple $\vec{t}$ in $V(G)^k_s$ 
by $\mathbf{U}^i_{T^k_s(G),v_\vec{t}}$. 
For notational convenience, we
write $\vec{U}^i_{T,v_\vec{t}}$ for $\mathbf{U}^i_{T^k_s(G),v_\vec{t}}$. Further, for two $(k,s)$-tuples $\vec{t}$ and $\vec{u}$, we write 
\begin{equation}\label{treeiso}
	\vec{U}^i_{T,v_\vec{t}} \simeq_{v_\vec{t} \to v_\vec{u}} \vec{U}^i_{T,v_\vec{u}}
\end{equation} 
if there exists a (labeled) isomorphism $\varphi$ between the two unrolled trees, mapping the (root) node $v_\vec{t}$ to $v_\vec{u}$. Moreover, we need the following two results. 
\begin{theorem}[\citep{Bus+1965, Val2002}]\label{tiso}
	The \wlone distinguishes any two directed, labeled non-isomorphic trees.
\end{theorem}
Using the first result, the second one states that the \wl{k}{s} can be simulated by the variant of the \wlone of~\cref{wlv} on the unrolled tree of the $(k,s)$-tuple graph, and hence can be reduced to a tree-isomorphism problem.
\begin{restatable}{lemma}{ktrees}
	\label{ktrees}
	Let $G$ be a \emph{connected} graph, then the \wl{k}{s} colors the tuples $\vec{t}$ and  $\vec{u}$ in $V(G)^{k}_s$ equally if and only if the corresponding unrolled $(k,s)$-tuple trees are isomorphic, i.e., 
	\begin{equation*}
		C^{k,s}_{i}(\vec{t}) = C^{k,s}_{i}(\vec{u}) \iff  \vec{U}^i_{T,v_\vec{t}}  \simeq_{v_\vec{t} \to v_\vec{u}} \vec{U}^i_{T,v_\vec{u}},
	\end{equation*}
	for all $i \geq 0 $. 
\end{restatable}
\begin{proof}[Proof sketch]
	First, by~\cref{lemma:wlk}, we can simulate the \wl{k}{s} for the graph $G$ using the $(k,s)$-tuple graph $T^k_s(G)$. Secondly, consider a node $v_\vec{t}$ in the $(k,s)$-tuple graph $T^k_s(G)$ and a corresponding node in the unrolled tree around $v_\vec{t}$. Observe that the neighborhoods for both nodes are identical. By definition, this holds for all nodes (excluding the leaves) in the unrolled tree. Hence, by \cref{lemma:wlk}, we can simulate the \wl{k}{s} for each tuple $\vec{t}$ by running the \wlone in the unrolled tree around $v_\vec{t}$ in the $(k,s)$-tuple graph. Since the  \wlone decides isomorphism for trees, see~\cref{tiso}, the result follows. 
\end{proof}

The following lemma shows that the \wl{k+1}{1} is strictly more expressive than the \wl{k}{1} for every $k \geq 2$. 
\begin{lemma}\label{lem:cyco}
	Let $k\geq 2$. Let $G \coloneqq C_{2(k+2)}$ and 
	$H \coloneqq C_{(k+2)} \,\dot\cup\, C_{(k+2)}$. 
	Then, the graphs $G$ and $H$ are distinguished by \wl{k+1}{1}, but they are not distinguished by \wl{k}{1}. 
\end{lemma}
\begin{proof}
	We first show that the \wl{k+1}{1} distinguishes the graphs $G$ and $H$.
	Let $\vec{v}$ in $V(G)^{k+1}_1$ be a tuple $(v_1,\dots,v_{k+1})$ such that $v_1,\dots,v_{k+1}$ is a path of length $k$ in $G$. 
	Let $\vec{w}$ in $V(H)^{k+1}_1$ be a tuple $(w_1,\dots,w_{k+1})$ such that $w_1,\dots,w_{k+1}$ is a path of length $k$ in $H$. By the structure of $H$, there exists a vertex $w_{k+2}$ in $V(H)$ such that $w_1,\dots,w_{k+2}$ forms a cycle of length $k+2$. 
	We claim that $\vec{v}$ does not have any local $1$-neighbor $\vec{x}$ in $V(G)^{k+1}_1$ such that $\vec{x}$ is non-repeating, i.e., every vertex in $\vec{x}$ is distinct. This holds because replacing the first vertex of $\vec{v}$ with any other vertex of $G$ will yield a disconnected tuple. On the other hand, $\vec{w}$ admits a non-repeating, local $1$-neighbor, obtained by replacing the first vertex $w_1$ by $w_{k+2}$. Hence, the \wl{k+1}{1} distinguishes $G$ and $H$.
	
	\begin{figure}[t]
		\begin{center}
			\subfloat[]{
				\centering
				\resizebox{!}{1.5cm}{\input{figures/g.tex}}}\qquad\qquad
			\subfloat[]{
				\centering
				\resizebox{!}{1.5cm}{\input{figures/h}}}
		\end{center}
		\caption{The graphs $A_{k+2}$ and $B_{k+2}$ for $k=4$.\label{fig:ab}}
	\end{figure}
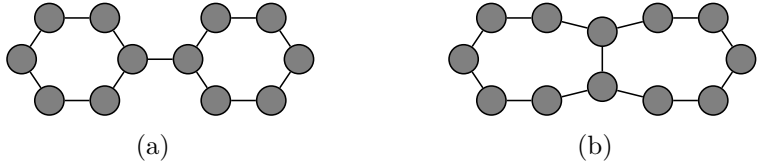
	
	Next, we show that the \wl{k}{1} does not distinguish the graphs $G$ and $H$. 
	Indeed, for every $j$ in $[k]$, every $k$-tuple $\vec{x}$ in $V(G)^k_1$ or $V(H)^k_1$ has exactly two local $j$-neighbors, corresponding to the two neighbors $y,z$ of the vertex $x_j$. The exact number of local $j$-neighbors of $\vec{x}$ which additionally lie in $V(G)^k_1$ (or $V(H)^k_1$) depends only on the atomic type of $\vec{x}$, since the length of cycles in $G$ and $H$ is at least $k+2$. Hence, the \wl{k}{1} neighborhood of every tuple in $G$ or $H$ depends only on its atomic type. This implies that the \wl{k}{1} does not refine the initial coloring for $G$ as well as $H$, and hence it does not distinguish $G$ and $H$. 
\end{proof}

Although~\cref{lem:cyco} already implies~\cref{theorem:one}, the construction hinges on the fact that the graphs $G$ and $H$ are not connected. To address this, for $k\geq 2$, we introduce two connected graphs $A_{k+2}$ and $B_{k+2}$ defined as follows. The graph $A_{k+2}$ has $2(k+2)$ nodes and $2(k+2)+1$ edges, and consists of two disjoint cycles on $k+2$ nodes connected by a single edge. The graph $B_{k+2}$ also has the same number of nodes and edges, and consists of two cycles on $k+3$ nodes, each, sharing exactly two adjacent nodes. See~\cref{fig:ab} for an illustration of the graphs $A_{k+2}$ and $B_{k+2}$ for $k=4$. We obtain the following result for the two graphs. 

\begin{lemma}\label{lem:cycox}
	For $k\geq 2$, the \wl{k}{1} does not distinguish the graphs $A_{k+2}$ and $B_{k+2}$, while the \wl{k+1}{1} does.
\end{lemma}
\begin{proof}
	We first show the second part, i.e., that the \wl{k+1}{1} distinguishes the graphs $A_{k+2}$ and $B_{k+2}$. Without loss of generality, assume that $V(A_{k+2}) = \{ a_1, \dots, a_{2(k+2)} \}$ and that $E(A_{k+2})$ consists of the edges $(a_i, a_{i+1})$ for $1 \leq i \leq k+1$, $(a_1, a_{(k+2)})$, $(a_i, a_{i+1})$ for $(k+3) \leq i \leq 2(k+2)-1$, and $(a_{k+3}, a_{2(k+2)})$. 
	The two cycles are connected by the edge $(a_{(k+2)}, a_{(k+3)})$ in $E(A_{k+2})$. Further, assume  $V(B_{k+2}) = \{ b_1, \dots, b_{2(k+2)} \}$ where $(b_i, b_{i+1})$ in $E(B_{k+2})$ for $1 \leq i \leq 2(k+2)-1$ and $(b_1, b_{2(k+2)+2})$ in $E(B_{k+2})$. Finally, the edge $(b_1, b_{k+3}) \in E(B_{k+2})$ is shared by the two cycles of length $k+3$ each.
	
	Now, let $\vec{t} = (a_1, \dots, a_{k+1})$ in $V(A_{k+2})^{k+1}_1$. Observe that the tuple $(a_1, \dots, a_{k+2})$ is a $(k+1)$-neighbor of the tuple $\vec{t}$, inducing a graph on $k+1$ nodes. Further, since the two cycles in the graph $B_{k+2}$ have length $k+3$, there is no tuple without repeated nodes that has a $(k+1)$-neighbor without repeated nodes. Hence, the two graphs are distinguished by \wl{k+1}{1}. 
	
	We now show that the \wl{k}{1} does not distinguish the graphs $A_{k+2}$ and $B_{k+2}$. First, we construct a bijection $\theta \colon V(A_{k+2}) \to V(B_{k+2})$ as induced by the following coloring:
	\begin{center}
		\resizebox{!}{1.5cm}{\input{figures/g_map.tex}}
	\end{center}
	Based on the bijection $\theta$, we define the bijection $\theta^k \colon V(A_{k+2})^k_1 \to (A_{B+2})^k_1$, by applying $\theta$ component-wise to $(k,s)$-tuples. Observe that $G[\vec{s}] \simeq G[\theta^k(\vec{s})]$  for $\vec{s}$ in $V(A_{k+2})^k_1$.  
	
	\begin{claim}\label{claimplus}
		Let $\vec{s}$ be a tuple in $V(G)^k_1$ and $\vec{t} = \theta^k(\vec{s})$ in $V(H)^k_1$. Let $N_j(\vec{s})$ and $N_j(\vec{t})$ be the $j$-neighbors of the  tuple $\vec{s}$ and  $\vec{t}$, respectively, for $j$ in $[k]$. 
		Then $\theta^k$ yields a one-to-one correspondence between $N_j(\vec{s})$ and $N_j(\vec{t})$.  
		Consequently, $G[\vec{u}] \simeq G[\theta^k(\vec{u})]$ for $\vec{u}$ in $N_j(\vec{s})$ and $\theta^k(\vec{u})$ in $N_j(\vec{t})$.
		
	\end{claim}
	\begin{proof}
	    The desired claim follows by observing that the bijective map $\theta\colon V(A_{k+2}) \to V(B_{k+2})$ preserves neighborhoods, i.e. for every $x$ in $V(A_{k+2})$, $\theta(N_{F}(x)) = N_K(\theta(x))$. 
	\end{proof}
	
	We now again leverage the above claim to show  that $C^{k,s}_i(\vec{s}) = C^{k,s}_i(\theta^k(\vec{s}))$ for $i \geq 0$, implying the required result. By a straightforward inductive argument, using~\cref{claimplus}, we can inductively construct a tree isomorphism between the unrolled trees  around the node $v_s$ and $v_t$ in the corresponding $(k,s)$-tuple graph such that $\vec{U}^i_{T,v_\vec{s}}  \simeq_{v_\vec{s} \to v_\vec{t}} \vec{U}^i_{T,v_\vec{t}}$. By~\cref{ktrees}, this implies $C^{k,s}_i(\vec{s}) = C^{k,s}_i(\theta(\vec{t}))$ for $i \geq 0$. 
	This shows that the \wl{k}{s} does not distinguish $A_{k+2}$ and $B_{k+2}$. 
\end{proof}

Hence, \cref{theorem:one} directly follows from \cref{lem:cycox}. Moreover, we also show that the \wl{k}{2} is more expressive than the \wl{k}{1}.
\begin{proposition}
	Let $k \geq 2$. Then 
	\begin{equation*}
		\text{\wl{k}{2}} \sqsubset \text{\wl{k}{1}}.
	\end{equation*}
\end{proposition}
\begin{proof}
	As in~\cref{lem:cyco}, let $G \coloneqq C_{2(k+2)}$ and $H \coloneqq C_{(k+2)} \,\dot\cup\, C_{(k+2)}$.
	By~\cref{lem:cyco}, $G$ and $H$ are not distinguished by the \text{\wl{k}{1}} for $k \geq 2$. We claim that the \wl{k}{2} distinguishes $G$ and $H$ for $k=2$ already. Since the \wl{k}{2} is at least as powerful as the \wl{2}{2}, this yields the desired claim. 
	
	With respect to the \wl{2}{2}, observe that the $(2,2)$-tuple graph $T^2_2(H)$ consists of four connected components while the $(2,2)$-tuple graph $T^2_2(G)$ consists of a single connected component. More precisely, there exist two connected components of $T^2_2(H)$ that consist only of $2$-tuples containing two non-adjacent nodes which are in the same connected component of the graph $H$. Note that none of these $2$-tuples is adjacent to any $2$-tuples in $V(H)^2_1$. Moreover, there exists no such connected component in $T^2_2(G)$. Also, note that the number of neighbors of each $2$-tuple of the graphs is exactly $4$, excluding self loops. Hence, the \wl{2}{2} will distinguish the two graphs.
\end{proof}

Moreover, the following result shows that increasing the parameter $s$ results in higher expressivity. Formally, we show that the \wl{k}{k} is strictly more expressive than the \wl{k}{2}. Note that we use vertex-colored graphs (rather than simple undirected graphs) in our proofs.

\begin{theorem}\label{thm:main}
	Let $k \geq 2$, then 
	\begin{equation*}
		\text{\wl{k}{k}} \sqsubset \text{\wl{k}{2}}.
	\end{equation*}
\end{theorem}

For the proof of \cref{thm:main}, we modify the construction employed in~\citep{Morris2020b}, Appendix C.1.1., where they provide an infinite family of graphs $(G_k, H_k)_{k \in \mathbb{N}}$ such that (a) \kwl does not distinguish $G_k$ and $H_k$, although (b) \localkwl distinguishes $G_k$ and $H_k$. Since our proof closely follows theirs, let us recall some relevant definitions from their paper. 

\emph{Construction of $G_k$ and $H_k$.} Let $K$ denote the complete graph on $k+1$ vertices (without any self-loops). The vertices of $K$ are indexed from $0$ to $k$. Let $E(v)$ denote the set of edges incident to $v$ in $K$: clearly, $|E(v)| = k$ for all $v$ in $V(K)$.
We call the elements of $V(K)$ \emph{base vertices}, and the elements of $E(K)$ \emph{base edges}.
Define the graph $G_k$ as follows:
\begin{enumerate}
	\item For the vertex set $V(G_k)$, we add   
	      \begin{enumerate}
	      	\item[(a)] $(v,S)$ for each $v$ in $V(K)$ and for each \emph{even} subset $S$ of $E(v)$, 
	      	\item[(b)] two vertices $e^1,e^0$ for each edge $e$ in $E(K)$.  
	      \end{enumerate} 
	\item For the edge set $E(G_k)$, we add 
	      \begin{enumerate}
	      	\item[(a)] an edge $\{e^0,e^1\}$ for each $e$ in $ E(K)$, 
	      	\item[(b)] an edge between $(v,S)$ and $e^1$ if $v$ in $ e$ and $e$ in $ S$,  
	      	\item[(c)] an edge between $(v,S)$ and $e^0$ if $v$ in $ e$ and $e$ not in $S$,  
	      \end{enumerate} 
\end{enumerate} 
For every $v$ in $ K$, the set of vertices of the form $(v,S)$ is called the \emph{vertex cloud} for $v$. Similarly, for every edge $e$ in $E(K)$, the set of vertices of the form $\{e^0,e^1\}$ is called the \emph{edge cloud} for $e$. 

Define a companion graph $H_k$, in a similar manner to $G_k$, with the following exception: in Step 1(a), for the vertex $0$ in $V(K)$, we choose all \emph{odd} subsets of $E(0)$. Counting vertices, we find that $|V(G)| = |V(H)| = (k+1)\cdot 2^{k-1} + \binom{k+1}{2} \cdot 2$. This finishes the construction of the graphs $G$ and $H$. We set $G_k \coloneqq G$ and $H_k \coloneqq H$. 

\emph{Distance-two-cliques.} A set $S$ of vertices is said to form a \emph{distance-two-clique} if the distance between any two vertices in $S$ is exactly $2$. The following results were shown in \citep{Morris2020b}.
\begin{lemma}[\citep{Morris2020b}]
	The following holds for the graphs $G_k$ and $H_k$ defined above. 
	\begin{itemize}
		\item There exists a distance-two-clique of size $(k+1)$ inside $G_k$.
		\item There does not exist a distance-two-clique of size $(k+1)$ inside $H_k$.
	\end{itemize}
	Hence, $G_k$ and $H_k$ are non-isomorphic. 
\end{lemma}
\begin{lemma}[\citep{Morris2020b}]\label{lem:neurips}
	The \localkwl distinguishes $G_k$ and $H_k$. 
	On the other hand, \kwl does not distinguish $G_k$ and $H_k$. 
\end{lemma}

We are ready to present the proof of \cref{thm:main}.

\begin{proof}[Proof of \cref{thm:main}]
	Observe that the \wl{k}{k} is the same as the \localkwl. Hence, it suffices to show an infinite family of graphs $(X_k,Y_k)$, $k$ in $\mathbb{N}$, such that (a) \wl{k}{2} does not distinguish $X_k$ and $Y_k$, although (b) \localkwl distinguishes $X_k$ and $Y_k$. 
	
	Let $X_k$ be the graph obtained from the graph $G_k$ as follows. First, for every base vertex $v$ in $V(K)$, every vertex of $V(G_k)$ in the vertex cloud for $v$ receives a color $\text{Red}_v$. Hence, vertex clouds form color classes, where each such class has a distinct color. Similarly, for every base edge $e$ in $E(K)$, every vertex of $V(G_k)$ in the edge cloud for $e$ receives a color $\text{Blue}_e$. Finally, let $\Delta > 3k$. Then, we replace every edge $e$ in $G_k$ by a path of length $\Delta$, such that every vertex on this path is colored with the color $(\{c,c'\})$, where $c$ and $c'$ are the colors of the endpoints of $e$ in $G_k$. We call such path vertices \emph{auxiliary} vertices. The graph $Y_k$ is obtained from $H_k$ by an identical construction. 
	
	First, we show that the \wl{k}{2} does not distinguish the graphs $X_k$ and $Y_k$. We use a modified version of the bijective $k$-pebble game \citep{Gro2017}: (a) we enforce the $k$ pebbles to form at most two components at any point during the game, and (b) when the Spoiler and Duplicator pick the $i^{th}$ pebble from each graph, the Duplicator is required to exhibit a bijection only between the position-$i$ local neighbourhoods of the two pebbling configuration tuples $\boldsymbol{x} \in X_k^k$ and $\boldsymbol{y} \in Y_k^k$, instead of a bijection between the vertex sets of $X_k$ and $Y_k$. 
	Observe that there can be at most two vertices out of $k+1$ vertices in $V(K)$ such that the corresponding vertex clouds contain a tupled vertex, by our choice of $\Delta$. Hence, in the usual parlance of Cai-Fürer-Immerman games \citep{Cai+1992}, the twisted edge can always be hidden among the remaining $(k-1)$ vertices of $K$. This ensures that for all $i \in [k]$, a partial isomorphism between $\boldsymbol{x}\backslash \boldsymbol{x}_i$ and $\boldsymbol{y}\backslash \boldsymbol{y}_i$ can always be extended to a bijective mapping between the $i$-local-neighborhoods of $\boldsymbol{x}$ and $\boldsymbol{y}$. Hence, the Duplicator cannot win this pebble game and therefore, $\wl{k}{2}$ cannot distinguish the graphs $X_k$ and $Y_k$.

	Next, we show that the \localkwl distinguishes the graphs $X_k$ and $Y_k$. Our proof closely follows the corresponding proof in \citep{Morris2020b}. Instead of showing a discrepancy in the number of distance-two-cliques, we instead use colored-distance-$(2\Delta+1)$-cliques defined as follows. Let $S$ be a set of vertices belonging to the vertex clouds. The set $S$ is said to form a colored-distance-$(2\Delta+1)$-clique if any two vertices in $S$ are connected by a path of exactly $2\Delta+1$ vertices, of which $2\Delta$ are auxiliary vertices and one vertex is a vertex from an edge cloud. Analogously to their proof, it can be shown that (a) there exists a colored-distance-$(2\Delta+1)$-clique of size $(k + 1)$ inside $X_k$, and (b) there does not exist a colored-distance-$(2\Delta+1)$-clique of size $(k + 1)$ inside $Y_k$, and hence, (c) $X_k$ and $Y_k$ are non-isomorphic. Finally, we claim that the \localkwl is powerful enough to detect colored-distance-$(2\Delta+1)$-cliques. The proof is analogous to \citep[Appendix C.1.1, Proof of Lemma 9]{Morris2020b}. This yields that the \localkwl distinguishes the graphs $X_k$ and $Y_k$. 
\end{proof}

\subsection{Asymptotic running time}\label{time_ext}

In the following, we bound the asymptotic running time of the \wl{k}{s}.
Due to~\cref{lemma:wlk}, we can upper-bound the running time of the \wl{k}{s} for a given graph by upper-bounding the time to construct the $(k,s)$-tuple graph and running the \wlone  variant of~\cref{wlv} on top. \cref{run1} establishes an upper bound on the asymptotic running time for constructing the $(k,s)$-tuple graph from a given graph. Thereto, we assume a \new{$d$-bounded-degree graph} $G$, for $d \geq 1$, i.e., each node has at most $d$ neighbors.

To prove the proposition, we define \new{$(k,s)$-multisets}. Let $G$ be a graph, $k \geq 1$, and $s$ in $[k]$, then the set of $(k,s)$-multisets
\begin{align*}
	S(G)^k_s = \{ \{\!\!\{ v_1, \dots, v_k \}\!\!\}  \mid \vec{v} \in V(G)^k_s \} 
\end{align*}
contains the set of multisets inducing subgraphs of $G$ on at most $k$ nodes with at most $s$ components. The following results upper-bounds the running time for the construction of $S(G)^k_s$.

\begin{algorithm}
	\caption{Generate $(k,s)$-multisets \label{algonew}}
	\begin{algorithmic}[1]
		\Require Graph $G$, $k$, $s$, and $S(G)^s_s$ 
		\Ensure $(k,s)$-multiset $S(G)^k_s$ 
		\State Let $R$ be an empty set data structure
		\For{$M \in S(G)^s_s$}
		\State Let $S$ be a queue data structure containing only $(M,s)$ 
		\While{$S$ not empty}
		\State Pop $(T,c)$ from queue $S$
		\If{$c+1 \leq k$}
		\For{$t \in T$}
		\For{$u \in \delta(t) \cup \{ t \}$}
		\State Add $(T \cup \{u\}, c+1)$ to $S$
		\EndFor
		\EndFor
		\Else
		\State Add $T$ to $R$
		\EndIf
		\EndWhile
		\EndFor
		\State \textbf{return} $R$
	\end{algorithmic}
\end{algorithm}

\begin{proposition}\label{run}
	Let $G$ be a $d$-bounded-degree graph, $k \geq 2$, and $s$ in $[k-1]$. Then \cref{algonew} computes $S^k_s(G)$ in time $\tilde\cO( n^s \cdot k^{k-s} (d+1)^{k-s})$.
\end{proposition}
\begin{proof}
	
	Let $c < k$ and let $T'$ be an element in $S(G)^{c+1}_s$. By definition of $S(G)^{c}_s$ and $S(G)^{c+1}_s$, there exists a $c$-element multiset $T$ in $S(G)^{c}_s$ such that $T' = T \cup \{ v \}$ is in $S(G)^{c+1}_s$ for a node $v$ in $V(G)$. Since $s$ is fixed, $v$ is either in the neighborhood $\delta(w)$ for $w$ in $T$ or $v = w'$ for $w'$ in $T$.  Hence, lines 7 to 9 in~\cref{algonew} generate  $S(G)^{c+1}_s$ from $S(G)^{c}_s$. The set data structure $R$ makes sure that the final solution will not contain duplicates. The running time follows directly when using, e.g., a red–black tree, to represent the set $R$.
\end{proof}

Based on the above result, we can easily construct  $T^k_s(G)$ from  $S^k_s(G)$, implying the following result. 
\begin{proposition}\label{run1}
	Let $G$ be a $d$-bounded-degree graph, $k \geq 3$, and $s$ in $[k-1]$. Then we can compute $T^k_s(G)$ in time $\tilde\cO( n^s \cdot k^{k-s} (d+1)^{k-s + 1} \cdot k! \cdot k)$.
\end{proposition}
\begin{proof}
	The running time follows directly from~\cref{run}. That is, from $S^k_s(G)$, we can generate the set $V^k_s(G)$ by generating all permutations of each element in the former. By iterating over each resulting $(k,s)$-tuple and each component of such $(k,s)$-tuple, we can construct the needed adjacency information. 
\end{proof}

Hence, unlike for the \kwl, the running time of the \wl{k}{s} does not depend on $n^k$ for an $n$-node graph and is solely dictated by $s$, $k$, and the sparsity of the graph.

Moreover, observe that the upper bound given in~\cref{run1}, by leveraging~\cref{lemma:wlk}, also upper-bounds the asymptotic running time for one iteration of the \wl{k}{s}. 

\section{SpeqNets: Sparse, permutation-equivariant graph networks}\label{gn}
We can now leverage the above combinatorial insights to derive sparsity-aware, permutation-equivariant graph networks, denoted \seq{k}{s}. Given a labeled graph $G$, let each $(k,s)$-tuple $\vec{v}$ in $V(G)^k_s$ be annotated with an initial feature $f^{(0)}(\vec{v})$ determined by its (labeled) isomorphism type, e.g., a one-hot encoding of $\tau_{G[\vec{v}]}$. Alternatively, we can also use some application-specific, real-valued feature.  In each layer $t > 0$,  we compute a new feature $f^{(t)}(\vec{v})$ as 
\begin{align}\label{gnngeneral}
    \begin{split}
    	f^{W_1}_{\text{mrg}}\Big(f^{(t-1)}(\vec{v}) ,f^{W_2}_{\text{agg}}\big( & \oms f^{(t-1)}(\phi_1(\vec{v},w)) \mid w \in \delta(v_1)  \text{ and } \phi_1(\vec{v},w) \in V(G)^k_s \cms, \dots, \\ &\oms f^{(t-1)}(\phi_k(\vec{v},w)) \mid w \in \delta(v_k) \text{ and } \phi_k(\vec{v},w) \in V(G)^k_s  \cms \big) \!\Big),
    \end{split}
\end{align}
in  $\bbR^{1 \times e}$, where $W_1^{(t)}$ and $W_2^{(t)}$ are learnable parameter matrices from $\bbR^{d \times e}$ for some $d, e > 0$. Here, $f^{W_2}_{\text{mrg}}$ and $f^{W_1}_{\text{agg}}$ are arbitrary differentiable functions, responsible for merging and aggregating the relevant feature information, respectively. Note that we can naturally handle discrete node and edge labels as well as directed graphs. The following result demonstrates the expressive power of the \seq{k}{s}, in terms of distinguishing non-isomorphic graphs. 
  
\begin{theorem}\label{equal}
	Let $(V, E, \ell)$ be a labeled graph, and let $k \geq 1$ and $s$ in $[k]$. Then for all \mbox{$t\geq 0$}, there exists weights $W_1^{(t)}$ and $W_2^{(t)}$ such that 
	\begin{equation*}
		C^{k,s}_{t}(\vec{v}) = C^{k,s}_{t}(\vec{w}) \iff f^{(t)}(\vec{v}) = f^{(t)}(\vec{w}).
	\end{equation*}
	Hence,
	the following holds for all $k \geq 1$:
	\begin{equation*}
		\text{\seq{k}{s}} \equiv \text{\wl{k}{s}}.
	\end{equation*}
\end{theorem} 
\begin{proof}[Proof sketch]
	First, observe that the \wl{k}{s} can be simulated on an appropriate node- and edge-labeled graph, see~\cref{lemma:wlk}. Secondly, following the proof of~\cite[Theorem 2]{Mor+2019}, there exists a parameter matrix $W_2^{(t)}$ such that we can injectively map each multiset in~\cref{gnngeneral}, representing the local $j$-neighbors for $j$ in $[k]$, to a $d$-dimensional vector. Moreover, we concatenate $j$ to each such vector to distinguish between different neighborhoods. Again, by~\cite[Theorem 2]{Mor+2019}, there exists a parameter matrix $W_1^{(t)}$ such that we can injectively map the set of resulting $k$ vectors to a unique vector representation. Alternatively, one can concatenate the resulting $k$ vectors and use a multi-layer perceptron to learn a joint lower-dimensional representation.
\end{proof}

It is not possible to come up with an architecture, i.e., instantiations of $f^{W_1}_{\text{mrg}}$ and  $f^{W_2}_{\text{agg}}$ such that it becomes more powerful than the \wl{k}{s}, see~\citep{Mor+2019}. \emph{However, all results from the previous section can be lifted to the neural setting,} see also~\cref{connect}.
Analogously to GNNs, the above architecture can naturally handle continuous node and edge labels. By using the tools developed in~\citep{Azi+2020}, it is straightforward to show that the above architecture is universal, i.e., it can approximate any possible permutation-invariant function over graphs up to an arbitrarily small additive error.

\subsection{Node-, edge-, and subgraph-level learning tasks}

The above architecture computes representation for $k$-tuples, rendering it mostly suitable for graph-level learning tasks, e.g., graph classification or regression. However, it is also possible to derive neural architectures based on the \wl{k}{s} for node- and edge-level learning tasks, e.g., node or link prediction. Given a graph $G$, to learn a node feature for node $v$, we can simply pool over the feature learned for $(k,s)$-tuples containing the node $v$ as a component. That is, let $t \geq 0$, then we consider the multisets 
\begin{equation}\label{nodepool}
	m^t(v)_i = \{\!\!\{ f^{(t-1)}(\vec{t}) \mid  \vec{t} \in V(G)^k_s \text{ and } t_i = v  \}\!\!\}
\end{equation}
for $i$ in $[k]$.
Hence, to compute a vectorial representation of the node $v$, we compute a vectorial representation of  $m^t(v)_i$ for $i$ in $[k]$, e.g., using a neural architecture for multi-sets, see~\citep{Wag+2021}, followed by learning a joint vectorial representation for the node $v$. Again, by~\citep{Azi+2020}, it is straightforward to show that the above architecture is universal, i.e., it can approximate any possible permutation-equivariant function over graphs up to an arbitrarily small additive error. Also, the above approach can be straightforwardly generalized to learn subgraph representations on an arbitrary number of vertices.

\section{Experimental evaluation}
Here, we aim to empirically investigate the learning performance of the  kernel, see~\cref{vr_ext}, and neural architectures, see~\cref{gn}, based on the \wl{k}{s}, compared with standard kernel and (higher-order) GNN baselines. Concretely, we aim to answer the following questions. 

\begin{description}
	\item[Q1] Do the \wl{k}{s}-based algorithms, both kernel and neural architectures, lead to improved classification and regression scores on real-world, graph-level benchmark datasets compared with dense algorithms and standard baselines?
	
	\item[Q2] How does the \seq{k}{s} architecture compare to standard GNN baselines on node-classification tasks?
	
	\item[Q3] To what extent does the \wl{k}{s} reduce computation times compared with architectures induced by the \kwl?
	
	\item[Q4] What is the effect of $k$ and $s$ with respect to computation times and predictive performance?
\end{description}

The source code of all methods and evaluation procedures is available at \url{https://www.github.com/chrsmrrs/speqnets}. 

\begin{table}[t!]
	\begin{center}
		\caption{Dataset statistics and properties for graph-level prediction tasks, $^\dagger$---Continuous vertex labels following~\cite{Gil+2017}, the last three components encode 3D coordinates.}
		\resizebox{1.0\textwidth}{!}{ 	\renewcommand{\arraystretch}{1.05}
			\begin{tabular}{@{}lcccccc@{}}\toprule
				\multirow{3}{*}{\vspace*{4pt}\textbf{Dataset}}&\multicolumn{6}{c}{\textbf{Properties}}\\
				\cmidrule{2-7}
				                         & Number of  graphs & Number of classes/targets & $\varnothing$ Number of nodes & $\varnothing$ Number of edges & Node labels              & Edge labels \\ \midrule
				$\textsc{Enzymes}$       & 600               & 6                         & 32.6                          & 62.1                          & \cmark                   & \xmark      \\
				$\textsc{IMDB-Binary}$   & 1\,000            & 2                         & 19.8                          & 96.5                          & \xmark                   & \xmark      \\
				$\textsc{IMDB-Multi}$    & 1\,500            & 3                         & 13.0                          & 65.9                          & \xmark                   & \xmark      \\
				$\textsc{Mutag}$         & 188               & 2                         & 17.9                          & 19.8                          & \cmark                   & \xmark      \\
								
				$\textsc{NCI1}$          & 4\,110            & 2                         & 29.9                          & 32.3                          & \cmark                   & \xmark      \\
				    
				$\textsc{PTC\_FM}$       & 349               & 2                         & 14.1                          & 14.5                          & \cmark                   & \xmark      \\
				$\textsc{Proteins}$      & 1\,113            & 2                         & 39.1                          & 72.8                          & \cmark                   & \xmark      \\
				$\textsc{Reddit-Binary}$ & 2\,000            & 2                         & 429.6                         & 497.8                         & \xmark                   & \xmark      \\ 
				\midrule
				$\textsc{Alchemy}$       & 202\,579          & 12                        & 10.1                          & 10.4                          & \cmark                   & \cmark      \\
				$\textsc{Qm9}$           & 129\,433          & 12                        & 18.0                          & 18.6                          & \cmark (13+3D)$^\dagger$ & \cmark (4)  \\
				\bottomrule
			\end{tabular}}
		\label{ds}
	\end{center}
\end{table}

\begin{table}[t!]
	\begin{center}
		\caption{Dataset statistics and properties for node-level prediction tasks.}
		\resizebox{.58\textwidth}{!}{ 	\renewcommand{\arraystretch}{1.05}
			\begin{tabular}{@{}lccc@{}}\toprule
				\multirow{3}{*}{\vspace*{4pt}\textbf{Dataset}}&\multicolumn{3}{c}{\textbf{Properties}}\\
				\cmidrule{2-4}
				                     & Number of nodes & Number of edges & Number of node features \\ \midrule
				$\textsc{Cornell}$   & 183             & 295             & 1\,703                  \\
				$\textsc{Texas}$     & 183             & 309             & 1\,703                  \\
				$\textsc{Wisconsin}$ & 251             & 490             & 1\,703                  \\
							
				\bottomrule
			\end{tabular}}
		\label{dss}
	\end{center}
\end{table}

\paragraph{Datasets} To compare the \wl{k}{s}-based kernels, we used the well-known graph-classification benchmark datasets from~\citep{Mor+2020}, see~\cref{ds} for dataset statistics and properties.\footnote{All datasets are publicly available at \url{www.graphlearning.io}.} To compare the \seq{k}{s} architecture with GNN baselines, we used the \textsc{Alchemy}~\citep{Che+2019b} and the \textsc{Qm9}~\citep{Ram+2014,Wu+2018} graph regression datasets, again see~\cref{t2} for dataset statistics and properties. Following~\cite{Morris2020b}, we opted for not using the 3D-coordinates of the \textsc{Alchemy} dataset to solely show the benefits of the (sparse) higher-order structures concerning graph structure and discrete labels. To investigate the performance of the architecture for node classification, we used the \textsc{WebKB} datasets~\citep{Pei+2020}, see~\cref{dss} for dataset statistics and properties.

\paragraph{Kernels}
We implemented the \wl{k}{s} and \wl{k}{s}$\!\!^+$ for $k$ in $\{2,3\}$ and $s$ in $\{1, 2 \}$. We compared our kernels to the Weisfeiler--Leman subtree kernel (\wlone)~\citep{She+2011}, the Weisfeiler--Leman Optimal Assignment kernel (\wloa)~\citep{Kri+2016}, the graphlet kernel (\gr)~\citep{She+2009}, and the shortest-path kernel~\citep{Bor+2005} (\shp). Further, we implemented the higher-order kernels \localkwl, \pluskwl, \deltakwl, and  \kwl kernel for $k$ in $\{2,3\}$ as outlined in~\citep{Morris2020b}. All kernels were (re-)implemented in \CC[11]. For the graphlet kernel, we counted (labeled) connected subgraphs of size 3. We followed the evaluation guidelines outlined in~\citep{Mor+2020}. 

\paragraph{Neural architectures} We used the \gineps and \textsf{Gin-$\varepsilon$-JK} architectures~\citep{Xu+2018b} as neural baselines. For data with (continuous) edge features, we used a $2$-layer MLP to map them to the same number of components as the node features and combined them using summation (\gine and \gineeps). For the evaluation of the \seq{k}{s} neural architectures of~\cref{gn}, we implemented them using \textsc{PyTorch Geometric}~\citep{Fey+2019}, using a  Python-wrapped \CC[11] preprocessing routine to compute the computational graphs for the higher-order GNNs. We used the \gineps layer to express $f^{W_1}_{\text{mrg}}$ and $f^{W_2}_{\text{agg}}$ of~\cref{gnngeneral}. 	
For the GNN baseline for the \textsc{Qm9} dataset, following~\citep{Gil+2017}, we used a complete graph, computed pairwise $\ell_2$ distances based on the 3D coordinates, and concatenated them to the edge features. We note here that our intent is not the beat state-of-the-art, physical knowledge-incorporating architectures, e.g., \textsf{DimeNet}~\citep{Kli+2020} or \textsf{Cormorant}~\citep{And+2019}, but to solely show the benefits of the local, sparse higher-order architectures compared to the corresponding ($1$-dimensional) GNN. For the \seq{k}{s} architectures, in the case of the \textsc{Qm9} dataset, to compute the initial features, for each $(k,s)$-tuple, we concatenated the node and edge features, computed pairwise $\ell_2$ distances based on the 3D coordinates, and a one-hot encoding of the (labeled) isomorphism type. Finally, we used a $2$-layer MLP to learn a joint, initial vectorial representation. For the node-classification experiments, we used mean pooling to implement~\cref{nodepool} and a standard \textsf{GCN} or \textsf{GIN} layer for all experiments, including the \seq{k}{s} architectures. Further, we used the architectures (\textsf{SDRF}) outlined in~\citep{Top+2021} as baselines. 

For the kernel experiments, we computed the (cosine) normalized Gram matrix for each kernel. We computed the classification accuracies using the $C$-SVM implementation of \textsc{LibSVM}~\citep{Cha+11}, using 10-fold cross-validation. We repeated each 10-fold cross-validation ten times with different random folds and report average accuracies and standard deviations. 

Following the evaluation method proposed in~\citep{Mor+2020}, the $C$-parameter was selected from $\{10^{-3}, 10^{-2}, \dotsc, 10^{2},$ $10^{3}\}$ using a validation set sampled uniformly at random from the training fold (using 10\% of the training fold). Similarly, the numbers of iterations of the
\wl{k}{s}, \wl{k}{s}$\!\!^+$, \wlone, \wloa, \localkwl, \pluskwl, and \kwl were selected from $\{0,\dotsc,5\}$ using the validation set. Moreover, for the \wl{k}{s}$\!\!^+$ and \pluskwl, we only added the label function $\#$ on the last iteration to prevent overfitting.
We report computation times for the \wl{k}{s}, \wl{k}{s}$\!\!^+$, \wloa, \localkwl, \pluskwl, and \kwl with five refinement steps. 

All kernel experiments were conducted on a workstation with 791\si GB of RAM using a single core. Moreover, we used the GNU \CC Compiler 4.8.5 with the flag \texttt{-O2}.

For comparing the kernel approaches to GNN baselines, we used 10-fold cross-validation and again used the approach outlined in~\citep{Mor+2020}. The number of components of the (hidden) node features in $\{ 32, 64, 128 \}$ and the number of layers in $\{ 1,2,3,4,5\}$ of the \gin and \gineps layer were again selected using a validation set sampled uniformly at random from the training fold (using 10\% of the training fold). We used mean pooling to pool the learned node embeddings to a graph embedding and used a $2$-layer MLP for the final classification, using a dropout layer with $p = 0.5$ after the first layer
of the MLP. We repeated each 10-fold cross-validation ten times with different random folds and report the average accuracies and standard deviations. Due to the different training methods, we do not provide computation times for the GNN baselines. 

For the larger molecular regression tasks \textsc{Alchemy} and \textsc{Qm9}, we closely followed the hyperparameters found in~\citep{Che+2019b} and~\citep{Gil+2017}, respectively, for the \gineeps layers. That is, we used six layers with 64 (hidden) node features and a set2seq layer~\citep{Vin+2016} for graph-level pooling, followed by a $2$-layer MLP for the joint regression of the twelve targets. We used the same architecture details and hyperparameters for the \seq{k}{s}. For the \textsc{Alchemy}, we used the subset of 12\,000 graphs from~\citep{Morris2020b}. For both datasets, we uniformly and at random sampled 80\% of the graphs for training, and 10\% for validation and testing, respectively. Moreover, following~\citep{Che+2019b,Gil+2017}, we normalized the targets of the training split to zero mean and unit variance. We used a single model to predict all targets. Following~\citep[Appendix C]{Kli+2020}, we report mean standardized MAE and mean standardized logMAE. We repeated each experiment five times and report average scores and standard deviations. We used the provided ten training, validation, and test splits for the node-classification datasets. All neural experiments were conducted on a workstation with one GPU card with 32GB of GPU memory.

To compare training and testing times between the $(2,1)$-\textsf{SpeqNet},  $(2,2)$-\textsf{SpeqNet}, \gineeps architectures, we trained all three models on \textsc{Alchemy (10k)} and \textsc{Qm9} to convergence, divided by the number of epochs, and calculated the ratio with respect to the average epoch computation time of the $(2,1)$-\textsf{SpeqNet} (average computation time of dense or baseline layer divided by average computation time of the $(2,1)$-\textsf{SpeqNet}). Contrary to the kernel timing experiments, we did not take into account the time of the preprocessing routine to compute the computational graphs to focus purely on the neural component of the architecture. Clearly, the time for the preprocessing of \seq{k}{s} with small $s$ is much smaller than that of, e.g., the \deltakwl.

\begin{table}[t]\centering			
	\caption{Classification accuracies in percent and standard deviations,  \textsc{Oot}--- Computation did not finish within one day, \textsc{Oom}--- Out of memory. The \wl{2}{2} and \wl{2}{2}$\!\!^+$ are omitted as they are equivalent to the $\delta$-$2$-LWL and $\delta$-$2$-LWL\xspace$\!\!^+$.}
	\label{t2}	
	\resizebox{.975\textwidth}{!}{ 	\renewcommand{\arraystretch}{1.05}
		\begin{tabular}{@{}c <{\enspace}@{}lcccccccc@{}}	\toprule
			& \multirow{3}{*}{\vspace*{4pt}\textbf{Method}}&\multicolumn{8}{c}{\textbf{Dataset}}\\\cmidrule{3-10}
			& & {\textsc{Enzymes}}         &  {\textsc{Imdb-Binary}}      & {\textsc{Imdb-Multi}}           & {\textsc{Mutag}}       & {\textsc{NCI1}}           & 
			{\textsc{Proteins}}         & {\textsc{PTC\_MR}}         &
			{\textsc{Reddit-Binary}  } \\	\toprule
			\multirow{4}{*}{\rotatebox{90}{\hspace*{-3pt}Baseline}}  & \gr                           & 29.9  \scriptsize	$\pm  0.8$        & 59.3     \scriptsize $\pm 0.9$          & 39.2 \scriptsize $\pm 0.6$            & 72.5 \scriptsize $\pm 1.7$           & 66.2 \scriptsize $\pm 0.2$          & 71.5   \scriptsize $\pm  0.5$        & 56.6 \scriptsize $\pm 1.3$          & 59.7 \scriptsize $\pm 0.5$          
			\\ 
			                                     & \shp                          & 40.3 \scriptsize	$\pm  0.9$         & 58.7     \scriptsize $\pm 0.6$          & 39.7 \scriptsize $\pm 0.3$            & 81.7   \scriptsize $\pm 1.5$         & 74.1  \scriptsize $\pm 0.2$         & 75.8 \scriptsize $\pm 0.7$           & 59.6  \scriptsize $\pm 1.5$         & 84.5 \scriptsize $\pm 0.2$          
			\\ 
			                                                         & \textsf{$1$-WL}               & 50.6  \scriptsize	$\pm 1.2 $        & 72.5    \scriptsize $\pm 0.8$           & 50.0 \scriptsize $\pm 0.8$            & 75.9 \scriptsize $\pm 2.0$           & 84.4 \scriptsize $\pm 0.3 $         & 73.1  \scriptsize $\pm 0.6$          & 59.3  \scriptsize $\pm 2.1$         & 73.4 \scriptsize $\pm 0.9$          
			\\ 	
			                                                         & \textsf{WLOA}                 & 57.1  \scriptsize	$\pm 0.8$         & 73.2 \scriptsize $\pm 0.4$              & 49.8 \scriptsize $\pm 0.4$            & 83.4 \scriptsize $\pm 1.2$           & 85.2  \scriptsize $\pm 0.2$         & 73.0 \scriptsize $\pm 0.9$           & 60.3   \scriptsize $\pm 1.9$        & 88.3 \scriptsize $\pm 0.4$   \\       
			\cmidrule{2-10}		
			\multirow{2}{*}{\rotatebox{90}{\hspace*{-3pt}GNN}}       & \textsf{Gin-$\varepsilon$}    & 38.7 \scriptsize	$\pm  1.5$         & 72.9    \scriptsize $\pm 0.7$           & 49.7 \scriptsize $\pm 0.7$            & 84.1 \scriptsize $\pm 1.4$           & 77.7 \scriptsize $\pm 0.8$          & 72.2  \scriptsize $\pm 0.6$          & 55.2 \scriptsize $\pm 1.7$          & 89.8 \scriptsize $\pm 0.4$          
			\\ 
			                                                         & \textsf{Gin-$\varepsilon$-JK} & 39.3  \scriptsize	$\pm 1.6 $        & 73.0  \scriptsize $\pm 1.1$             & 49.6 \scriptsize $\pm 0.7$            & 83.4 \scriptsize $\pm 2.0$           & 78.3   \scriptsize $\pm 0.3 $       & 72.2 \scriptsize $\pm 0.7$           & 56.0  \scriptsize $1.3\pm $         & 90.4 \scriptsize $\pm 0.4$          
			\\ 
			\cmidrule{2-10}	
			\multirow{6}{*}{\rotatebox{90}{\kwl}} 	&
			\textsf{$2$-WL}       &  37.0 \scriptsize	$\pm 1.0$ &   68.1  \scriptsize $\pm 1.7$ & 47.5 \scriptsize $\pm 0.7$ &  85.7 \scriptsize $\pm 1.6$ &  66.9 \scriptsize $\pm 0.3$ &    75.2 \scriptsize $\pm 0.4$ & 60.5  \scriptsize $\pm 1.1$ &\textsc{Oom} 
			\\ 
			                                                         & \textsf{$3$-WL}               & 42.3   \scriptsize	$\pm 1.1$        & 67.1    \scriptsize $\pm 1.5$           & 46.8  \scriptsize $\pm 0.8$           & 85.4  \scriptsize $\pm 1.5$          & \textsc{Oot}                        & \textsc{Oot}                         & 59.0 \scriptsize $\pm 2.0$          & \textsc{Oom}                        
			\\      
			\cmidrule{2-10}			
					
			                                                         & \textsf{$\delta$-$2$-LWL}     & 55.9  \scriptsize	$\pm 1.0$         & 73.0     \scriptsize $\pm 0.7$          & 50.1 \scriptsize $\pm 0.9$            & 85.6  \scriptsize $\pm 1.4$          & 84.6  \scriptsize $\pm 0.3$         & 75.1 \scriptsize $\pm 0.5$           & 61.7 \scriptsize $\pm 2.4 $         & 89.4  \scriptsize $\pm 0.6$         
			\\    		
			                                                         & \textsf{$\delta$-$2$-LWL\xspace$\!\!^+$} & 53.9  \scriptsize	$\pm 1.4$         & 75.6  \scriptsize $\pm 1.0$             & 62.7 \scriptsize $\pm 1.4$            & 84.1  \scriptsize $\pm 2.1$          & \textbf{91.3} \scriptsize $\pm 0.3$ & 79.2 \scriptsize $\pm 1.2$           & 61.6  \scriptsize $\pm 1.3$         & 91.4 \scriptsize $\pm 0.4$          
			\\ 
			                                                         & \textsf{$\delta$-$3$-LWL}     & \textbf{58.2} \scriptsize	$\pm 1.2$ & 72.6     \scriptsize $\pm 0.9 $         & 49.0 \scriptsize $\pm 1.2$            & 84.1   \scriptsize $\pm 1.6$         & 83.2  \scriptsize $\pm 0.4$         & \textsc{Oom}                         & 60.7 \scriptsize $\pm 2.2  $        & \textsc{Oom}                        
			\\ 
			                                                         & \textsf{$\delta$-$3$-LWL\xspace$\!\!^+$} & 56.5 \scriptsize	$\pm 1.4$          & 76.1  \scriptsize $\pm 1.2$             & 64.3 \scriptsize $\pm 1.2$            & 85.4   \scriptsize $\pm 1.8$         & 82.7  \scriptsize $\pm 0.4$         & \textsc{Oom}                         & 61.5   \scriptsize $\pm 1.8$        & \textsc{Oom}                        
			\\ 
						
			\cmidrule{2-10}	
						
			\multirow{6}{*}{\rotatebox{90}{\hspace*{-3pt}\wl{k}{s}}} & \wl{2}{1}                     & 53.7   \scriptsize	$\pm 1.7$        & 73.5     \scriptsize $\pm 0.8$          & 50.8 \scriptsize $\pm 0.7$            & 84.2  \scriptsize $\pm 1.7$          & 82.8  \scriptsize $\pm 0.3$         & 73.2  \scriptsize $\pm 0.6$          & 55.9  \scriptsize $\pm 2.4$         & 76.9  \scriptsize $\pm 0.6$         
			\\ 
			                                                         & \wl{2}{1}$\!\!^+$             & 51.6 \scriptsize	$\pm 1.8$          & 73.7    \scriptsize $\pm 1.1$           & 55.4 \scriptsize $\pm 0.9$            & 79.6   \scriptsize $\pm 3.4$         & 81.9 \scriptsize $\pm 0.3 $         & 76.0  \scriptsize $\pm 0.9$          & 60.2 \scriptsize $\pm 2.1 $         & \textbf{94.7} \scriptsize $\pm 0.3$ 
			\\ 
			                                                         & \wl{3}{1}                     & 53.4 \scriptsize	$\pm 1.4$          & 74.6   \scriptsize $\pm 1.0$            & 51.3 \scriptsize $\pm 0.6$            & 85.3  \scriptsize $\pm 2.4$          & 81.4 \scriptsize $\pm 0.5$          & 72.9  \scriptsize$\pm 1.1$           & 60.2  \scriptsize $\pm 1.7$         & \textsc{Oom}                        
			\\
			                                                         & \wl{3}{1}$\!\!^+$             & 57.0  \scriptsize	$\pm  1.9$        & \textbf{87.1}     \scriptsize $\pm 0.6$ & \textbf{67.1}  \scriptsize $\pm 1.1 $ & 79.2 \scriptsize $\pm 1.5$           & 89.8  \scriptsize $\pm 0.4$         & \textbf{81.2}  \scriptsize $\pm 0.8$ & 59.2 \scriptsize $\pm 2.0$          & \textsc{Oom}                        
			\\ 
			                                                         & \wl{3}{2}                     & 56.4 \scriptsize	$\pm 0.7$          & 73.5    \scriptsize $\pm 0.5$           & 49.7 \scriptsize $\pm 0.6$            & \textbf{86.4}  \scriptsize $\pm 2.6$ & 84.9 \scriptsize $\pm 0.4$          & 75.1 \scriptsize $\pm 0.9$           & 61.9  \scriptsize $\pm 2.4$         & \textsc{Oom}                        \\ 
			                                                         & \wl{3}{2}$\!\!^+$             & 55.8 \scriptsize	$\pm 1.7$          & 78.1   \scriptsize $\pm 1.4$            & 59.5 \scriptsize $\pm 1.0$            & 84.5 \scriptsize $\pm 1.9$           & 89.4  \scriptsize $\pm 0.3$         & 78.8 \scriptsize $\pm 0.6$           & \textbf{62.3} \scriptsize $\pm 3.3$ & \textsc{Oom}                        \\ 
			\bottomrule
		\end{tabular}}
\end{table}

\begin{table}[t]
	\caption{Additional experimental results for graph regression and node classification.}
	\label{t2ll}
	\centering		
	\subfloat[][Mean MAE (mean std.\ MAE, logMAE) on large-scale (multi-target) molecular regression tasks.]{
		\resizebox{0.49\textwidth}{!}{ 	\renewcommand{\arraystretch}{1.05}
			\begin{tabular}{@{}lcc@{}}	\toprule
				\multirow{3}{*}{\vspace*{4pt}\textbf{Method}}&\multicolumn{2}{c}{\textbf{Dataset}}\\\cmidrule{2-3}
				                  & {\textsc{Alchemy (10k)}}                                            & {\textsc{Qm9}}                                                     \\	\toprule
				\gineeps          & 0.180 {\scriptsize $\pm  0.006$} -1.958  {\scriptsize $\pm  0.047$} & 0.079 {\scriptsize $\pm 0.003$} -3.430  {\scriptsize $\pm 0.080$ } \\
				\cmidrule{1-3}
						
				\text{\seq{2}{1}}    &  0.169
				{\scriptsize $\pm 0.005$} -2.010    {\scriptsize $\pm 0.056$}     & 0.078 {\scriptsize $\pm  0.007$} -2.947  \scriptsize $\pm 0.171 $   \\
				\text{\seq{2}{2}}    & \textbf{0.115}
				{\scriptsize $\pm 0.001$} -2.722  {\scriptsize $\pm 0.054$}  &  \textbf{0.029} {\scriptsize $\pm 0.001$}  -4.081  \scriptsize $\pm 0.058$  \\
				\cmidrule{1-3}
				\text{\seq{3}{1}} & 0.180 {\scriptsize $\pm 0.011$} -1.914 {\scriptsize $\pm 0.097$}    & 0.068 {\scriptsize $\pm 0.003$} -3.397  {\scriptsize $\pm 0.086$}  \\
				\text{\seq{3}{2}}    & \textbf{0.115} 
				{\scriptsize $\pm 0.002$} -2.767  {\scriptsize $\pm 0.079$}      & \textsc{Oot}   \\

				\bottomrule
			\end{tabular}}
		
	}
	\qquad
	\subfloat[][Classification accuracies and standard deviations for node classification.]{
		\resizebox{0.427\textwidth}{!}{ 	\renewcommand{\arraystretch}{1.05}
			\begin{tabular}{@{}lccc@{}}	\toprule
				\multirow{3}{*}{\vspace*{4pt}\textbf{Method}}&\multicolumn{3}{c}{\textbf{Dataset}}\\\cmidrule{2-4}
				                  & {\textsc{Cornell}}                      & {\textsc{Texas}}              & {\textsc{Wisconsin}} \\	\toprule
				\textsf{GCN} & 56.5 {\scriptsize $\pm 0.9$}  & 58.2 {\scriptsize 
				$\pm 0.8$}  & 50.9
				{\scriptsize $\pm 0.7$}\\ 
 	             \textsf{GIN} & 51.9 {\scriptsize $\pm 1.1$}  & 55.3  {\scriptsize $\pm 2.7$}  & 48.4  {\scriptsize $\pm 1.6$}\\ 
				\textsf{SDRF} + Undirected & 57.5 {\scriptsize $\pm 0.3$}  & \textbf{70.4} {\scriptsize 
				$\pm 0.6$}  & 61.6
				{\scriptsize $\pm 0.9$}\\ 
				\cmidrule{1-4}
				\text{\seq{2}{1}}    & 63.9 
				{\scriptsize $\pm 1.7$}&66.8  {\scriptsize $\pm 0.9$} & 
				67.7  {\scriptsize $\pm 2.2$}  \\
				\text{\seq{2}{2}} & \textbf{67.9}   {\scriptsize $\pm 1.7$} & 67.3  {\scriptsize $\pm 2.0$} & \textbf{68.4}        
				{\scriptsize $\pm 2.2$} \\
				\text{\seq{3}{1}} & 61.8  {\scriptsize $\pm 3.3$}           & 68.3 {\scriptsize $\pm 1.3$}  &                      
				60.4 {\scriptsize $\pm 2.8$} \\
				\bottomrule
			\end{tabular}}
	}

\end{table}

\subsection{Results and discussion}
In the following, we answer questions \textbf {Q1} to \textbf{Q4}.

\paragraph{A1} \textit{Kernels}  See~\cref{t2}. The \wl{k}{s} for $k,s$ in $\{ 2,3\}$ significantly improves the classification accuracy compared to the \kwl and the \deltakwl, and the other kernel baselines, while being on par with or better than the $\delta$-$2$-\textsf{LWL} and $\delta$-$3$-\textsf{LWL}. The \wl{k}{s} and \wl{k}{s}$\!\!^+$ achieve a new state-of-the-art on five out of eight datasets. Our algorithms also perform vastly better than the neural baselines. 

\textit{Neural architectures} See~\cref{t2ll}. On both datasets, all \seq{k}{s} architectures beat the GNN baseline. On the \textsc{Alchemy} dataset, the \seq{2}{2} and \seq{3}{1} perform best, while on the \textsc{Qm9} dataset, the \seq{2}{2} performs best by a large margin.

\paragraph{A2} See~\cref{t2ll}. Over all three datasets, the \seq{k}{s} architectures improve over the \textsf{GCN} baseline. Specifically, over all datasets, the \seq{2}{1} and the \seq{2}{2} lead to an increase of at least 7\% in accuracy. For example, both architectures beat the \textsf{GCN} baseline by at least 17\% on the \textsc{Wisconsin} dataset. Further, the \seq{k}{s} architectures lead to better accuracies compared to the \textsf{SDRF} architecture, e.g., improving on it by more than 10\% on the \textsc{Cornell} dataset. Hence, node-level tasks also benefit from higher-order information. However, increasing $k$ more does not result in increased accuracies.

\begin{table}[t]\centering	\renewcommand{\arraystretch}{1.05}
	\caption{Overall computation times for selected datasets in seconds  (Number of iterations for WL-based methods: 5), \textsc{Oot}---Computation did not finish within one day (24h), \textsc{Oom}---Out of memory.}
	\label{t1_app}
	\resizebox{0.73\textwidth}{!}{ 	\renewcommand{\arraystretch}{1.05}
		\begin{tabular}{@{}c <{\enspace}@{}lcccccc@{}}	\toprule
			& \multirow{3}{*}{\vspace*{4pt}\textbf{Graph Kernel}}&\multicolumn{6}{c}{\textbf{Dataset}}\\\cmidrule{3-8}
			                                                         &                               & {\textsc{Enzymes}} & {\textsc{IMDB-Binary}} &  {\textsc{IMDB-Multi}}& {\textsc{Mutag}} & {\textsc{NCI1}} & {\textsc{PTC\_MR}} \\	\toprule
			\multirow{4}{*}{\rotatebox{90}{\hspace*{-3pt}}}          & \textsf{$1$-WL}               & $<$1               & $<$1        &    $<$1         & $<$1  & 1.9         &    $<$1     \\
					
			\cmidrule{2-8}	
			\multirow{2}{*}{\rotatebox{90}{Glob.}} 	&
			\textsf{$2$-WL}    &   225.9     &   91.2    & 38.3 & 4.3 &  1\,127.8   &  10.7    \\
			                                                         & \textsf{$3$-WL}               &    55\,242.7         &      17\,565.2    &  4\,977.1  & 259.8 & \textsc{Oot}  & 1324.2 \\

			\cmidrule{2-8}
			\multirow{4}{*}{\rotatebox{90}{Local}}    		
						
			                                                         & \textsf{$\delta$-$2$-LWL}     & 25.2            &    27.3      & 19.8  & $<$1  & 82.2  & 1.1  \\
			                                                         & \textsf{$\delta$-$2$-LWL\xspace$\!\!^+$} &      25.6       &    26.1     & 18.5 & $<$1  & 108.3 & 1.2    \\       
						
			                                                         & \textsf{$\delta$-$3$-LWL}     &   3\,519.0              &  3\,560.3     &   1957.5 & 36.5 & 15\,207.3  & 89.9  \\
			                                                         & \textsf{$\delta$-$3$-LWL\xspace$\!\!^+$} &         3\,674.9    &     3\,636.5     &  2162.3 & 43.6 & 15\,945.6  & 111.1   \\
			\cmidrule{2-8}	
						
			\multirow{6}{*}{\rotatebox{90}{\hspace*{-3pt}\wl{k}{s}}} & \wl{2}{1}                   &      1.6       &    12.0      & 11.0 & $<$1 & 5.8  & $<$1     \\
			                                                         & \wl{2}{1}$\!\!^+$            &     1.7       &    12.6       & 11.0 & $<$1 & 6.7 & $<$1   \\
			                              & \wl{3}{1}   &         51.1         &       1\,040.2               & 1112.2 & 1.4 & 111.3  & 1.9     \\
			                                                         & \wl{3}{1}$\!\!^+$             &   52.1          &     1\,049.4   &  1238.7  & 1.6 & 120.1  & 2.0 \\
			                                                         & \wl{3}{2} &             937.9      &      2\,571.1       & 2252.6 &   19.0     & 3\,502.6  & 29.6  \\
			                                                         & \wl{3}{2}$\!\!^+$            &   1\,046.1          &    2\,937.8      &  2572.2 & 22.4  & 3\,888.7 & 34.4 \\
			\bottomrule
		\end{tabular}}
\end{table}

\begin{table}[t]\centering
	\caption{\label{t2len}Average speed-up ratios over all epochs (training and testing).}
	\resizebox{.32\textwidth}{!}{\renewcommand{\arraystretch}{1.05}
		\begin{tabular}{@{}lcc@{}}
			\toprule
			\multirow{3}{*}{\vspace*{4pt}\textbf{Method}} &       \multicolumn{2}{c}{\textbf{Dataset}}       \\
			\cmidrule{2-3}           & {\textsc{Alchemy (10k)}} & {\textsc{Qm9}} \\ \toprule

			\gineeps                 & 0.5                      & 1.3            \\
			$(2,1)$-\textsf{SpeqNet} & 1.0                      & 1.0            \\
			$(2,2)$-\textsf{SpeqNet} & 1.3                      & 3.4            \\ \bottomrule
		\end{tabular}}
\end{table}

\paragraph{A3} \textit{Kernels} See~\cref{t1_app}. Clearly, for the same $k$ and $s < k$, the \wl{k}{s} improves over the \kwl and its (local) variants. For example, on the \textsc{Enzymes} dataset, the  \wl{2}{1} is more than 20 times faster in terms of computation times compared to the $\delta$-$2$-\textsf{LWL}. The speed-up is even more significant for the (non-local) $2$-\textsf{WL}. This speed-up factor increases more as $k$ increases, e.g., the \wl{3}{1} is more than 1\,700 times faster compared to the $3$-\textsf{WL}, whereas the \wl{3}{2} is still more than 87 times faster,  while giving better accuracies. Similar speed-up factors can be observed over all datasets. 

\textit{Neural architectures} See~\cref{t2len}.
The \seq{2}{1} severely speeds up the computation time across both datasets. Specifically, on the \textsc{Alchemy} dataset, the \seq{2}{1} is 1.3 times faster compared to the \seq{2}{2}, while requiring twice the computation time of the \gineeps but achieving a lower MAE. More interestingly, on the \textsc{Qm9} dataset, the \seq{2}{1} is 3.4 times faster compared to the \seq{2}{2}, while also being 1.3 times faster compared to the \gineeps. The speed-up over \gineeps is most likely due to the latter considering the complete graph to compute all pairwise $\ell_2$ distances, whereas the \seq{2}{1} only considers connected node pairs.

\paragraph{A4} See~\cref{t2,t2ll}. The \wl{2}{1}, \wl{2}{1}$\!\!\!^+$, and \wl{3}{1}$\!\!\!^+$ beat the \wlone on six out of eight datasets. Going from the \wl{3}{1} to the \wl{3}{2} often leads to a slight increase in accuracy, e.g., on the \textsc{Enzymes} and \textsc{Mutag} datasets, while sometimes leading to a drop in accuracy, e.g., on the \textsc{Imdb-Binary} dataset. Hence, a better understanding of the model's generalization performance with respect to $s$ needs to be investigated in future work. The effect is less pronounced for the neural architectures; however, all higher-order models beat the GNN baseline. Reducing $s$ leads to a vast reduction in computation time. For example, on the \textsc{Enzymes} dataset, going from the \wl{3}{2} to the \wl{3}{1} leads to a speed-up factor of more than 18, while only inducing a small drop in terms of accuracy, whereas the \text{\wl{3}{1}$\!\!^+$} beats the \wl{3}{2} while only increasing the computation time by one second. Similar observations can be made across all datasets.

Increasing $s$ often leads to a better performance on the graph regression tasks. For example, on the \textsc{Alchemy} dataset, going from a \seq{2}{1} to a \seq{2}{2} architecture reduces the MAE by 0.054. Similar effects can be observed for the \textsc{Qm9} dataset, and when going from a  \seq{3}{1} to a \seq{3}{2} architecture. On the node-classification datasets, reducing $s$ leads to a slight drop in accuracy, between 0.5 and 4\%, while increasing $k$ beyond $2$ often results in a drop in accuracy.

\section{Conclusion}
To circumvent the exponential running time requirements of the \kwl, we introduced a new heuristic for the graph isomorphism problem, namely the \wl{k}{s}. By varying the parameters $k$ and $s$, the \wl{k}{s} offers a trade-off between scalability and expressivity and, unlike the \kwl, takes into account the potential sparsity of the graph. Based on these combinatorial insights, we designed provably expressive machine-learning architectures, $(k,s)$-\textsf{SpeqNets}, suitable for node-, subgraph-, and graph-level prediction tasks. Empirically, we showed that such architectures lead to state-of-the-art results in node- and graph-level classification regimes while obtaining promising results on graph-level regression tasks. \emph{We believe that this principled approach paves the way for designing new permutation-equivariant architectures, overcoming the limitations of current graph neural networks.}

\section*{Acknowledgements}
Christopher Morris is in part supported by the German Academic Exchange Service (DAAD) through a DAAD IFI postdoctoral, a DFG Emmy Noether grant (468502433), and  RWTH Junior Principal Investigator Fellowship under 
the Excellence Strategy of the Federal Government and the Länder. Gaurav Rattan is supported by the DFG Research Grants Program–RA 3242/1-1–411032549. Siamak Ravanbakhsh's research is in part supported by CIFAR AI Chairs program. Computational resources were provided by Mila.

\setcitestyle{numbers}
\bibliography{bibliography}

\end{document}

%% file: figures/v.tex
\begin{tikzpicture}[node distance = {.5cm}, v/.style = {draw, circle},  minimum size=25pt,every node/.style={inner sep=0pt,outer sep=0}]
  \node (a) [v] {$a$};
  \node (b) [v, left = of a] {$b$};
  \node (c) [v, right = of a] {$c$};
  \node (d) [v, above = of a] {$d$};

  \draw (a) to node[above] {} (c);
  \draw (a) to node[above] {} (b);
  \draw (a) to node[above] {} (d);

  \draw (d) to node[above] {} (b);
\end{tikzpicture}

%% file: figures/u.tex
\begin{tikzpicture}[node distance = {.5cm}, v/.style = {draw, circle}, minimum size=25pt,every node/.style={inner sep=0pt,outer sep=0}]
  \node (a) [v] {$a_{(0,a)}$};
  \node (b) [v, below  = .4cm of a] {$c_{(1,a)}$};
\node (c) [v, below right = 1.cm of a] {$b_{(1,a)}$};
\node (d) [v, below left = 1.cm of a] {$d_{(1,a)}$};

\node (e) [v, below =  .4cm of d] {$a_{(2,d)}$};
\node (f) [v, below left =1.0cm of d] {$b_{(2,d)}$};
\node (g) [v, below = .4cm of b] {$a_{(2,c)}$};
\node (i) [v, below  =  .4cm of c] {$a_{(2,b)}$};
\node (j) [v, below right = 1.0cm of c] {$d_{(2,b)}$};

\draw[->] (a) to node[above] {} (b);
\draw[->] (a) to node[above] {} (c);
\draw[->] (a) to node[above] {} (d);

\draw[->] (d) to node[above] {} (e);
\draw[->] (d) to node[above] {} (f);
\draw[->] (b) to node[above] {} (g);

\draw[->] (c) to node[above] {} (i);
\draw[->] (c) to node[above] {} (j);
\end{tikzpicture}

%% file: figures/g.tex
\begin{tikzpicture}

\tikzset{line/.style={draw,thick}}
\tikzset{arrow/.style={line,->,>=stealth}}
\tikzset{node/.style={circle,inner sep=0pt,minimum width=15pt}}

\node[line,node,fill=gray] (x1)  at (0.75,  0.00) {};
\node[line,node,fill=gray] (x2)  at (1.25,  0.75) {};
\node[line,node,fill=gray] (x3)  at (2.25,  0.75) {};
\node[line,node,fill=gray] (x4)  at (2.75,  0.00) {};
\node[line,node,fill=gray] (x5)  at (2.25, -0.75) {};
\node[line,node,fill=gray] (x6)  at (1.25, -0.75) {};
\node[line,node,fill=gray] (x7)  at (3.75,  0.00) {};
\node[line,node,fill=gray] (x8)  at (4.25,  0.75) {};
\node[line,node,fill=gray] (x9)  at (5.25,  0.75) {};
\node[line,node,fill=gray] (x10) at (5.75,  0.00) {};
\node[line,node,fill=gray] (x11) at (5.25, -0.75) {};
\node[line,node,fill=gray] (x12) at (4.25, -0.75) {};

\path[line] (x1) to (x2);
\path[line] (x2) to (x3);
\path[line] (x3) to (x4);
\path[line] (x4) to (x5);
\path[line] (x5) to (x6);
\path[line] (x6) to (x1);
\path[line] (x4) to (x7);
\path[line] (x7) to (x8);
\path[line] (x8) to (x9);
\path[line] (x9) to (x10);
\path[line] (x10) to (x11);
\path[line] (x11) to (x12);
\path[line] (x12) to (x7);

\end{tikzpicture}

%% file: figures/h.tex
\begin{tikzpicture}

\tikzset{line/.style={draw,thick}}
\tikzset{arrow/.style={line,->,>=stealth}}
\tikzset{node/.style={circle,inner sep=0pt,minimum width=15pt}}

\node[line,node,fill=gray] (x1)  at (0.75,   0.00) {};
\node[line,node,fill=gray] (x2)  at (1.25,   0.75) {};
\node[line,node,fill=gray] (x3)  at (2.25,   0.75) {};
\node[line,node,fill=gray] (x4)  at (3.25,   0.50) {};
\node[line,node,fill=gray] (x5)  at (3.25,  -0.50) {};
\node[line,node,fill=gray] (x6)  at (2.25,  -0.75) {};
\node[line,node,fill=gray] (x7)  at (1.25,  -0.75) {};
\node[line,node,fill=gray] (x8)  at (4.25,  0.75) {};
\node[line,node,fill=gray] (x9)  at (5.25,  0.75) {};
\node[line,node,fill=gray] (x10) at (5.75,  0.00) {};
\node[line,node,fill=gray] (x11) at (5.25, -0.75) {};
\node[line,node,fill=gray] (x12) at (4.25, -0.75) {};

\path[line] (x1) to (x2);
\path[line] (x2) to (x3);
\path[line] (x3) to (x4);
\path[line] (x4) to (x5);
\path[line] (x5) to (x6);
\path[line] (x6) to (x7);
\path[line] (x7) to (x1);
\path[line] (x4) to (x8);
\path[line] (x8) to (x9);
\path[line] (x9) to (x10);
\path[line] (x10) to (x11);
\path[line] (x11) to (x12);
\path[line] (x12) to (x5);

\end{tikzpicture}

%% file: figures/g_map.tex
\begin{tikzpicture}

\tikzset{line/.style={draw,thick}}
\tikzset{arrow/.style={line,->,>=stealth}}
\tikzset{node/.style={circle,inner sep=0pt,minimum width=15pt}}

\node[] (x1)  at (0.75,  0.00) {\large\textbf{\dots}};
\node[line,node,fill=blue] (x2)  at (1.25,  0.75) {};
\node[line,node,fill=green] (x3)  at (2.25,  0.75) {};
\node[line,node,fill=yellow] (x4)  at (2.75,  0.00) {};
\node[line,node,fill=pink] (x5)  at (2.25, -0.75) {};
\node[line,node,fill=gray] (x6)  at (1.25, -0.75) {};
\node[line,node,fill=purple] (x7)  at (3.75,  0.00) {};
\node[line,node,fill=white] (x8)  at (4.25,  0.75) {};
\node[line,node,fill=olive] (x9)  at (5.25,  0.75) {};
\node[] (x10) at (5.75,  0.00) {\large\textbf{\dots}};
\node[line,node,fill=lime] (x11) at (5.25, -0.75) {};
\node[line,node,fill=orange] (x12) at (4.25, -0.75) {};

\path[line] (x1) to (x2);
\path[line] (x2) to (x3);
\path[line] (x3) to (x4);
\path[line] (x4) to (x5);
\path[line] (x5) to (x6);
\path[line] (x6) to (x1);
\path[line] (x4) to (x7);
\path[line] (x7) to (x8);
\path[line] (x8) to (x9);
\path[line] (x9) to (x10);
\path[line] (x10) to (x11);
\path[line] (x11) to (x12);
\path[line] (x12) to (x7);

\node[] (xx1)  at (8.75,   0.00) {\large\textbf{\dots}};
\node[line,node,fill=blue] (xx2)  at (9.25,   0.75) {};
\node[line,node,fill=green] (xx3)  at (10.25,   0.75) {};
\node[line,node,fill=yellow] (xx4)  at (11.25,   0.50) {};
\node[line,node,fill=purple] (xx5)  at (11.25,  -0.50) {};
\node[line,node,fill=white] (xx6)  at (10.25,  -0.75) {};
\node[line,node,fill=olive] (xx7)  at (9.25,  -0.75) {};
\node[line,node,fill=pink] (xx8)  at (12.25,  0.75) {};
\node[line,node,fill=gray] (xx9)  at (13.25,  0.75) {};
\node[] (xx10) at (13.75,  0.00) {\large\textbf{\dots}};
\node[line,node,fill=lime] (xx11) at (13.25, -0.75) {};
\node[line,node,fill=orange] (xx12) at (12.25, -0.75) {};

\path[line] (xx1) to (xx2);
\path[line] (xx2) to (xx3);
\path[line] (xx3) to (xx4);
\path[line] (xx4) to (xx5);
\path[line] (xx5) to (xx6);
\path[line] (xx6) to (xx7);
\path[line] (xx7) to (xx1);
\path[line] (xx4) to (xx8);
\path[line] (xx8) to (xx9);
\path[line] (xx9) to (xx10);
\path[line] (xx10) to (xx11);
\path[line] (xx11) to (xx12);
\path[line] (xx12) to (xx5);

\end{tikzpicture}